\documentclass[conference]{IEEEtran}
\pdfoutput=1
\IEEEoverridecommandlockouts
\usepackage{cite}
\usepackage{amsmath,amssymb,amsfonts}
\usepackage{algorithmic}
\usepackage{graphicx}
\usepackage{textcomp}
\usepackage{xcolor}
\def\BibTeX{{\rm B\kern-.05em{\sc i\kern-.025em b}\kern-.08em
    T\kern-.1667em\lower.7ex\hbox{E}\kern-.125emX}}

\usepackage{array}
\usepackage{amsthm}

\DeclareMathOperator*{\argmin}{arg\,min}
\usepackage{fancyhdr}
\usepackage[shortlabels]{enumitem}

\usepackage{mathtools}

\usepackage{letltxmacro}
\LetLtxMacro\orgvdots\vdots
\LetLtxMacro\orgddots\ddots

\makeatletter
\DeclareRobustCommand\vdots{%
  \mathpalette\@vdots{}%
}
\newcommand*{\@vdots}[2]{%
  \sbox0{$#1\cdotp\cdotp\cdotp\m@th$}%
  \sbox2{$#1.\m@th$}%
  \vbox{%
    \dimen@=\wd0 %
    \advance\dimen@ -3\ht2 %
    \kern.5\dimen@
    \dimen@=\wd2 %
    \advance\dimen@ -\ht2 %
    \dimen2=\wd0 %
    \advance\dimen2 -\dimen@
    \vbox to \dimen2{%
      \offinterlineskip
      \copy2 \vfill\copy2 \vfill\copy2 %
    }%
  }%
}
\DeclareRobustCommand\ddots{%
  \mathinner{%
    \mathpalette\@ddots{}%
    \mkern\thinmuskip
  }%
}
\newcommand*{\@ddots}[2]{%
  \sbox0{$#1\cdotp\cdotp\cdotp\m@th$}%
  \sbox2{$#1.\m@th$}%
  \vbox{%
    \dimen@=\wd0 %
    \advance\dimen@ -3\ht2 %
    \kern.5\dimen@
    \dimen@=\wd2 %
    \advance\dimen@ -\ht2 %
    \dimen2=\wd0 %
    \advance\dimen2 -\dimen@
    \vbox to \dimen2{%
      \offinterlineskip
      \hbox{$#1\mathpunct{.}\m@th$}%
      \vfill
      \hbox{$#1\mathpunct{\kern\wd2}\mathpunct{.}\m@th$}%
      \vfill
      \hbox{$#1\mathpunct{\kern\wd2}\mathpunct{\kern\wd2}\mathpunct{.}\m@th$}%
    }%
  }%
}
\makeatother

\fancypagestyle{firstpage}{
  \fancyhf{}

  \fancyhead[C]{\normalsize{To Appear in 2021 ACM/IEEE 48th Annual International Symposium on Computer Architecture (ISCA)}} 
  \fancyfoot[C]{\thepage}
}  

\fancypagestyle{restpage}{
  \fancyhf{}

  \fancyhead[C]{\normalsize{ }} 
  \fancyfoot[C]{\thepage}
}  

\begin{document}

\title{RingCNN: Exploiting Algebraically-Sparse Ring Tensors for Energy-Efficient CNN-Based Computational Imaging
\thanks{This work was supported by the Ministry of Science and Technology, Taiwan, R.O.C., under Grant no. MOST 109-2218-E-007-034.}
}

\author{\IEEEauthorblockN{Chao-Tsung Huang}
\IEEEauthorblockA{\textit{Department of Electrical Engineering} \\
\textit{National Tsing Hua University}\\
HsinChu, Taiwan, R.O.C. \\
chaotsung@ee.nthu.edu.tw}
}

\maketitle
\thispagestyle{firstpage}
\pagestyle{restpage}

\begin{abstract}
In the era of artificial intelligence, convolutional neural networks (CNNs) are emerging as a powerful technique for computational imaging.
They have shown superior quality for reconstructing fine textures from badly-distorted images and have potential to bring next-generation cameras and displays to our daily life.
However, CNNs demand intensive computing power for generating high-resolution videos and defy conventional sparsity techniques when rendering dense details.
Therefore, finding new possibilities in regular sparsity is crucial to enable large-scale deployment of CNN-based computational imaging.

In this paper, we consider a fundamental but yet well-explored approach---\textit{algebraic sparsity}---for energy-efficient CNN acceleration.
We propose to build CNN models based on \textit{ring algebra} that defines multiplication, addition, and non-linearity for $n$-tuples properly.
Then the essential sparsity will immediately follow, e.g. $n$-times reduction for the number of real-valued weights.
We define and unify several variants of ring algebras into a modeling framework, \textit{RingCNN}, and make comparisons in terms of image quality and hardware complexity.
On top of that, we further devise a novel ring algebra which minimizes complexity with component-wise product and achieves the best quality using \textit{directional ReLU}.
Finally, we design an accelerator, \textit{eRingCNN}, to accommodate to the proposed ring algebra, in particular with regular ring-convolution arrays for efficient inference and on-the-fly directional ReLU blocks for fixed-point computation.
We implement two configurations, $n=2$ and $4$ (50\% and 75\% sparsity), with 40~nm technology to support advanced denoising and super-resolution at up to 4K UHD 30 fps.
Layout results show that they can deliver \textit{equivalent} 41~TOPS using 3.76 W and 2.22 W, respectively.
Compared to the real-valued counterpart, our ring convolution engines for $n=2$ achieve 2.00$\times$ energy efficiency and 2.08$\times$ area efficiency with similar or even better image quality.
With $n=4$, the efficiency gains of energy and area are further increased to 3.84$\times$ and 3.77$\times$ with only 0.11~dB drop of peak signal-to-noise ratio (PSNR).
The results show that RingCNN exhibits great architectural advantages for providing near-maximum hardware efficiencies and graceful quality degradation simultaneously.
\end{abstract}


\section{Introduction}

Convolutional neural networks (CNNs) have demonstrated their superiority in the fields of computer vision and computational imaging.
The former includes object recognition \cite{vggnet_2015,resnet_2016} and detection \cite{detection_review_2019}.
The latter involves image denoising \cite{DnCNN_2017,FFDNet_2018}, super-resolution (SR) \cite{VDSR_2016,SRResNet_2017,EDSR_2017,WDSR_2019}, and style transfer \cite{st_sr_2016,cyclegan_2017}; in particular, denoising is the key to enhance low-light photography on mobile phones, and SR plays an important role for displaying lower-resolution contents on ultra-high-resolution (UHD) TVs.
Although CNNs can be applied in both application fields, the computation schemes are quite different and so are their design challenges.

Recognition and detection CNNs aim to extract high-level features and usually process small images with a huge amount of parameters.
In contrast, computational imaging ones need to generate low-level and high-precision details and often deal with much larger images with fewer parameters.
For example, the state-of-the-art FFDNet for denoising \cite{FFDNet_2018} requires only 850K weights but can be used to generate 4K UHD videos at 30 fps.
This will demand as high as 106 TOPS (tera operations per second) of computation, and the precision of multiplications could be at least 8-bit for representing sufficient dynamic ranges.
Therefore, for computational imaging it is the intensive computation for rendering \textit{fine-textured}, \textit{high-throughput}, and \textit{high-precision} feature maps to pose challenges for the deployment in consumer electronics.

Exploiting sparsity in computation is a promising way to reduce complexity for CNNs.
Many approaches have been analyzed in detail, but most of them are discussed only for recognition and detection.
The most common one is to explore natural sparsity for feature maps \cite{cnvlutin_2016,diffy_2018} and/or filter weights \cite{scnn_2017,sparten_2019,smartexchange_2020}.
Utilizing such sparsity, like unstructured pruning \cite{deepcompression_2016}, will induce computation irregularity and thus significant hardware overheads.
For example, the state-of-the-art SparTen \cite{sparten_2019} only delivers 0.43~TOPS/W on 45~nm technology for the dedicated designs to tame irregularity.
If, instead, structured pruning \cite{sparsegran_2017,rewind_prune_2020} is applied to improve regularity, model quality will then drop quickly. 
Thus, natural sparsity is hard to support high-throughput inference with low power consumption.

Another common approach is to explore the low-rank sparsity in over-parameterized CNNs by either decomposition \cite{tt_lowrank_2011,cp_lowrank_2015,tucker_lowrank_2016} or model structuring \cite{mobilenet_2017,squeezenet_2017,mobilenetv2_2018,tie_2019}.
It aims high compression ratios and approximates weight tensors by regular but radically-changed inference structures.
This low-rank approximation works well for recognition CNNs which extract high-level features.
But it could quickly deteriorate the representative power of computational imaging ones for generating local details.
For example, merely applying depth-wise convolution can lead to 1.2~dB of peak signal-to-noise ratio (PSNR) drop for SR networks \cite{eCNN_2019}.
Therefore, low-rank sparsity may not be suitable for fine-textured CNN inference.

A recent alternative for providing regular acceleration is to enforce full-rank sparsity on matrix-vector multiplications \cite{circnn_2017,shufflenet_2018,hadanet_2019}.
It partitions them into several $n\times n$ sub-block multiplications and then replaces each one by a component-wise product between $n$-tuples.
This is equivalent to a group convolution with data reordering \cite{hadanet_2019}; therefore, for restoring representative power additional pre-/post-processing is required to mix information between components or groups.
CirCNN \cite{circnn_2017} equivalently applies Fourier transform on each sub-block for this purpose by forcing weight matrices to be block-circulant.
ShuffleNet \cite{shufflenet_2018} instead performs global channel shuffling, and HadaNet \cite{hadanet_2019} adopts simpler Hadamard transform.
However, the applicability of this approach is unclear for computational imaging because CirCNN aims very high compression ratios (66$\times$ for AlexNet \cite{alexnet_2012}), and ShuffleNet and HadaNet focus only on bottleneck convolutions.

Lastly, a fundamental but less-discussed approach is to exploit \textit{algebraic sparsity}.
In contrast to using real numbers, CNNs can also be constructed by complex numbers \cite{dcn_2018} or quaternions \cite{dqn_2018,qcnn_2019,qcnn_eccv_2018}.
By their nature, the number of real-valued weights can decrease two or four times, respectively.
Moreover, their multiplications can be accelerated by fast algorithms.
For example, the quaternion multiplication is usually expressed by a $4\times 4$ real-valued matrix and can be simplified into eight real-valued multiplications and some linear transforms \cite{q_complexity_1975}.
Regarding activation functions, the real-valued component-wise rectified linear unit (ReLU) is mostly adopted, and its efficiency over complex-domain functions is demonstrated in \cite{dcn_2018}.
Since this algebraic sparsity can reduce complexity with moderate ratios and high regularity, it is a good candidate for accelerating computational imaging.
However, previous work only discusses the two traditional division algebras and thus poses strict limitations for implementation.

In this paper, we would like to lay down a more generalized framework---\textit{RingCNN}---for algebraic sparsity to expand its design space for model-architecture co-optimization.
Observing that division is usually not required by CNN inference, we propose to construct models by \textit{ring}, a fundamental algebraic structure on $n$-tuples with definitions of multiplication and addition.
In particular, we consider a bilinear formulation for ring multiplication to have transform-based fast algorithms and thus include full-rank sparsity into this framework.
For constructing CNN models, we also equip non-linearity to the rings.
Then several ring variants are defined properly and compared systematically for joint quality-complexity optimization.

This algebraic generalization also brings architectural insights on ring non-linearity.
We observe that conventional methods mostly adopt the component-wise ReLU for non-linearity and use the linear transforms in ring multiplication for information mixing.
However, for fixed-point implementation these transforms will increase input bitwidths for the following component-wise products and bring significant hardware overheads. 
Inspired by this, we propose a ring with a novel \textit{directional ReLU} to serve both non-linearity and information mixing.
Then we can avoid the transforms before the products to eliminate the bitwidth-increasing overheads.
Extensive evaluations will show that the proposed ring can achieve not only the best hardware efficiency for multiplications but also the best image quality for its compact structure for training.

Finally, we design an accelerator---eRingCNN---to utilize the proposed ring for high-throughput and energy-efficient CNN acceleration.
For comparison purposes, we adopt eCNN \cite{eCNN_2019}, the state-of-the-art for computational imaging, as our architecture backbone.
Then we devise highly-parallel ring-convolution engines for efficient inference and simply replace the real-valued counterparts in eCNN thanks to their regularity and similarity in computation. 
For the directional ReLU which involves two transforms, conventional MAC-based accelerators may need to perform quantization before each transform and cause up to 0.2~dB of PSNR drop.
Instead, we apply an on-the-fly processing pipeline to avoid unnecessary quantization errors and facilitate fixed-point inference on 8-bit features.
With 40~nm technology, we implement two sparsity settings, $n=2$ and $4$, for eRingCNN to show the effectiveness.

In summary, the main contributions and findings of this paper are:
\begin{itemize}
\item We propose a novel modeling framework, RingCNN, to thoroughly explore algebraic sparsity.
The corresponding training process, including quantization, is also established for in-depth quality comparisons.
\item We propose a novel ring variant with a directional ReLU which achieves better image quality and area saving than complex field, quaternions, the rings alike to CirCNN and HadaNet, and all newly-discovered ones.
Its image quality even outperforms unstructured weight pruning and sometimes, when $n=2$, can be better than real field.
\item We design and implement accelerators with two configurations: eRingCNN-n2 (50\% sparsity) and eRingCNN-n4 (75\%).
They can deliver \textit{equivalent} 41~TOPS using only 3.76~W and 2.22~W, respectively, and support high-quality computational imaging at up to 4K UHD 30~fps.
\item Our ring convolution engines achieve near-maximum hardware efficiencies ($\cong n$).
Layout results show that for $n=2$ they have 2.00$\times$ energy efficiency and 2.08$\times$ area efficiency compared to the real-valued counterpart.
Those for $n=4$ can increase the corresponding efficiency gains to 3.84$\times$ and 3.77$\times$, respectively.
\item RingCNN models provide competitive image quality.
Compared to the real-valued models for eCNN, those for eRingCNN-n2 even have an average PSNR gain of 0.01~dB and those for eRingCNN-n4 only drop by 0.11~dB.
When serving Full-HD applications on eRingCNN, they can outperform the advanced FFDNet \cite{FFDNet_2018} for denoising and SRResNet \cite{SRResNet_2017} for SR.

\end{itemize}

\section{Motivation}
\label{sec:motivation}

We aim to enable next-generation computational imaging on consumer electronics by achieving high-throughput and high-quality inference with energy-efficient and cost-effective acceleration.
However, computational imaging CNNs require dense model structures to generate fine-textured details.
Thus before deploying any complexity-reducing method we need to examine the impact of image quality and the gain of computation complexity as a whole.

Without loss of generality, we demonstrate this quality-complexity tradeoff using the advanced model SRResNet as an example.
In Fig. \ref{fig:fig_motivation}, two conventional sparsity techniques are examined.
One is unstructured magnitude-based weight pruning for exploring natural sparsity.
It shows graceful quality degradation when compression ratios are up to 2$\times$, 4$\times$, and 8$\times$.
However, its irregular computation will erode the performance gain due to induced hardware overheads and load imbalance.
For example, only 11.7\% of power consumption and 5.6\% of area are spent on MACs in the sparse tensor accelerator SparTen \cite{sparten_2019}.
The other examined technique is depth-wise convolution (DWC) which exploits low-rank sparsity.
The quality drops very quickly and even can be worse than the old-fashioned VDSR \cite{VDSR_2016}.
As a result, weight pruning and DWC are unfavourable for computational imaging due to the computation irregularity and the quality distortion respectively.

A more straightforward approach is to reduce the model size in a compact way, and here we consider two cases: shrinking either model depth or feature channels.
For SRResNet, the depth reduction causes sharp quality loss.
In contrast, the channel reduction provides a good quality-complexity tradeoff which shows a similar trend as weight pruning and performs much better than DWC.
In particular, this approach maintains high computation regularity and can be accelerated by energy-efficient dense tensor accelerators, such as eCNN \cite{eCNN_2019} in which 94.0\% of power consumption and 72.8\% of area are spent on convolutions.
Therefore, the compact model configurations should also be considered before applying sparsity.

In this paper, we would like to explore the possibility of having the quality of weight pruning and the regularity of compact modeling at the same time.
We will approach this goal by using ring algebra for the elementary operations in CNNs.
In this way, we can achieve \textit{local sparsity} and assure \textit{global regularity} simultaneously.
Our results, RingCNN, for SRResNet are also shown in Fig. \ref{fig:fig_motivation} to demonstrate the effectiveness.
The details of our approach will be introduced in the following.

\begin{figure}[t]
\centering
\includegraphics[width=8.7cm]{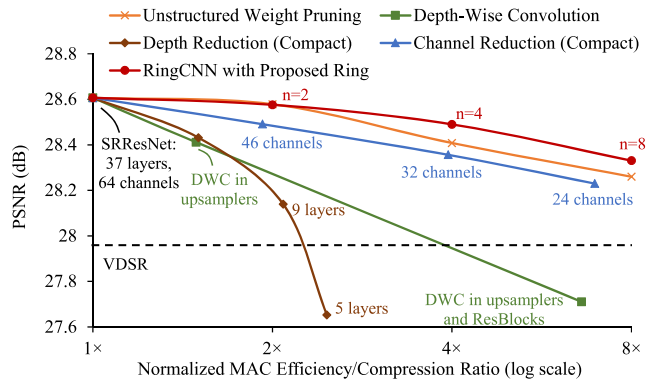} 
\caption{Computation efficiency versus image quality.
Different complexity-reducing methods are applied to SRResNet for four-times SR tasks (scaling up by four times for both of image width and height).
The models are trained using the same training strategy.
The image quality is measured by the averaged PSNR over test datasets Set5 \cite{set5}, Set14 \cite{set14}, BSD100 \cite{cbsd68}, and Urban100 \cite{urban100}.
}
\label{fig:fig_motivation}
\end{figure}

\section{Ring Algebra for Neural Networks}
\label{sec:ring_albegra}

Deep neural networks consist of many feed-forward layers.
These layers are usually defined over real field $\mathbb{R}$ and formulated by its three elementary operations: addition $+$, multiplication $\cdot\,$, and non-linearity $f$.
With tensor extensions, we can have a common formulation for each $l$-th layer:
\begin{align}
\bold{x}^{(l)}=\bold{f}^{(l)}(\bold{G}^{(l)}\bold{x}^{(l-1)}+\bold{b}^{(l)}),
 \label{equ:dnn_form}
\end{align}
where $\bold{x}^{(l-1)}$, $\bold{G}^{(l)}$, $\bold{b}^{(l)}$, and $\bold{f}^{(l)}$ represent the input feature tensor, weight tensor, bias tensor, and non-linear tensor operation respectively.
Conversely, as long as we define the three operations properly, we can construct neural networks at will using other algebraic structures.

An example for using complex field is shown in Fig. \ref{fig:fig_ring_example}.
Each complex number $z$ can be expressed by either a complex form $z_0+z_1 i$ or an equivalent 2-tuple
$\left(\begin{smallmatrix}
z_0\\
z_1
\end{smallmatrix}\right)$.
Then weight storage can be reduced by a half, and arithmetic computation can be accelerated by the complex multiplication algorithm, i.e.~the complexity for each complex multiplication can be reduced from four real multiplications to three.
In the following, we will first consider ring algebras to generalize this idea and define proper ring multiplication for discussion.
Then we will analyze their demands of hardware resources, and, finally, propose a novel ring variant with directional non-linearity to maximize hardware efficiency.

\begin{figure*}
\centering
\includegraphics[width=16cm]{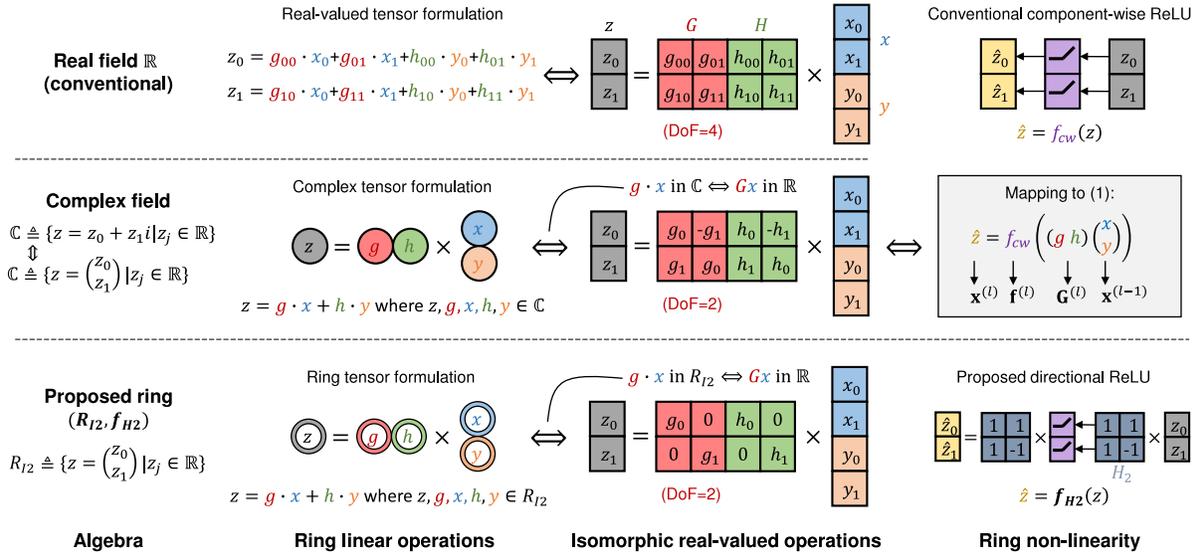} 
\caption{A simple neural-network layer for four real-valued inputs ($x_0,x_1,y_0,$ and $y_1$) and two real-valued outputs ($\hat{z}_0$ and $\hat{z}_1$)
by (top) real field $\mathbb{R}$, (middle) complex field $\mathbb{C}$, and (bottom) a proposed 2-tuple ring $(R_{I2},f_{H2})$.
For the latter two algebras, the inputs are equivalently two 2-tuples ($x$ and $y$) and the output becomes one 2-tuple ($\hat{z}$).
Their tensor formulations, $z=
\left(\begin{smallmatrix}
g & h 
\end{smallmatrix}\right)
\left(\begin{smallmatrix}
x \\
y 
\end{smallmatrix}\right)
\mbox{ in } \mathbb{C} \mbox{ or } R_{I2}$, have isomorphic operations in real field: $z=Gx+Hy \mbox{ in } \mathbb{R}$.
Then the degrees of freedom (DoF) in real numbers for each weight sub-matrix, e.g. $G$, are reduced from four ($g_{00},g_{01},g_{10},$ and $g_{11}$) to two ($g_0$ and $g_1$).
}
\label{fig:fig_ring_example}
\end{figure*}

\subsection{Ring Algebra}
\label{ssec:ring_algebra}

A ring $R$ is a fundamental algebraic structure which is a set equipped with two binary operations $+$ and $\cdot\,$.
Here we consider the set of real-valued $n$-tuples, i.e.\ $R=\left\lbrace x=(x_0,...,x_{n-1})^t \mid x_i \in \mathbb{R} \right\rbrace$.
For clarity, $x$ is a ring \textit{element} and $x_i$ is its real-valued \textit{component}.
And we simply use the component-wise vector addition for the ring addition $+$.

As for ring multiplication $\cdot\,$, it plays an important role for the properties of different rings.
Given
\begin{align}
z = g \cdot x
 \label{equ:ring_mul}
\end{align}
where $z,g,x\in R$, we consider it has a bilinear form to have a general formulation for fast algorithms, which will be discussed in Section \ref{ssec:fast_ringmul}.
In particular, the components of the three ring elements are related by
\begin{align}
z_i = \sum_{j=0}^{n-1} \sum_{k=0}^{n-1} M_{ikj} g_k  x_j,
 \label{equ:ring_bilinear}
\end{align}
where $M$ is a 3-D indexing tensor with only $1$, $0$, and $-1$ as its entries.
In other words, the products of input ring components in form of $g_k x_j$ are distributed to output components $z_i$ through $M_{ikj}$.
With the bilinear form (\ref{equ:ring_bilinear}), the ring multiplication (\ref{equ:ring_mul}) will be isomorphic to a matrix-vector multiplication
\begin{align}
z = G x,
 \label{equ:ring_iso_mul}
\end{align}
where the matrix $G$ has entries $G_{ij}=\sum_{k=0}^{n-1} M_{ikj} g_k $.
Without loss of generality, we will use $g$ for filter weights and $x$ for feature maps in the following.

After having definitions of $+$ and $\cdot\,$, we still need to define a unary non-linear operation $f$ for a ring to construct neural networks.
A conventional choice, which is usually adopted by previous methods for full-rank or algebraic sparsity, is a component-wise ReLU
\begin{align}
f_{cw}(x)=(\max(0,x_0),...,\max(0,x_{n-1}))^t, \label{equ:fcw}
\end{align}
where $\max(0,\cdot)$ is the commonly-used real-valued ReLU.

\subsection{Fast Ring Multiplication}
\label{ssec:fast_ringmul}

Now we will integrate transform-based full-rank sparsity into this framework.
For the bilinear-form ring multiplication (\ref{equ:ring_bilinear}), from \cite{winograd_book_1980} we know that its optimal general fast algorithm over real field can be expressed by the following three steps
\begin{align}
\mbox{filter/data transform: }&\tilde{g}=T_g g, \tilde{x}=T_x x, \label{equ:ring_fast1} \\ 
\mbox{component-wise product: }& \tilde{z}=\tilde{g} \circ \tilde{x} \mbox{ (on }m\mbox{-tuples)}, \label{equ:ring_fast2} \\ 
\mbox{reconstruction transform: }& z=T_z \tilde{z},  \label{equ:ring_fast3}
\end{align}
where $T_g$ and $T_x$ are $m\times n$ transform matrices for $g$ and $x$ respectively, and $T_z$ is $n\times m$ for $z$.
And $\circ$ represents a component-wise product, i.e.~$\tilde{z}_i=\tilde{g}_i \tilde{x}_i$ for $i=0,1,...,m-1$ for the three $m$-tuples $\tilde{z}$, $\tilde{g}$, and $\tilde{x}$.
If the transform matrices involve only simple coefficients, e.g.~$\pm 1$ or $0$, then they can be implemented by adders, and the component-wise product will dominate computation complexity.
In particular, the number of real-valued multiplications can be reduced from the general $n^2$ for matrix $G$ to $m$ in (\ref{equ:ring_fast2}).
Therefore, the complexity of fast ring multiplication depends on how we decompose the indexing tensor $M$ or its isomorphic matrix $G$ into (\ref{equ:ring_fast1})-(\ref{equ:ring_fast3}).

When $G$ is diagonalizable over real field, i.e.~$G=T^{-1} D T$, this complexity can be minimized as $m=\mbox{rank}(G)$ for $\tilde{g}=\mbox{diag}(D)$.
The proof is given in Appendix \ref{ssec:decomp_g}.
In this perspective, a ring $R_H$ alike to HadaNet has a full-rank $G$, i.e.~$\mbox{rank}(G)=n$, which is diagonalized by Hadamard transform.
Another example is a ring $R_I$ equivalent to group convolution which applies component-wise products for a diagonal full-rank $G$, and its invertible $T$ is simply the identity matrix $I$.

In contrast, if $G$ is not diagonlizable over real field, we can instead apply the tensor rank decomposition for the indexing tensor $M$ as mentioned in \cite{q_complexity_1975}.
However, the complexity is usually larger than $\mbox{rank}(G)$ in this case, and the generic rank (grank) represents the lower bound for real-valued multiplications: 
$m \geq \mbox{grank}(M)$.
For example, the rotation matrix for complex field $\mathbb{C}$ leads to three real-valued multiplications as its $\mbox{grank}(M)=3$ while $\mbox{rank}(G) = 2$, and the circulant matrix in CirCNN also belongs to this category.
The related properties of $\mathbb{C}$ as well as $R_H$ and $R_I$ with ring dimension $n=2$ are as shown at the top of Table \ref{tab:tab_ring_properties}.

\begin{table*}
\centering
\caption{Properties of ring algebras.}
\includegraphics[width=18cm]{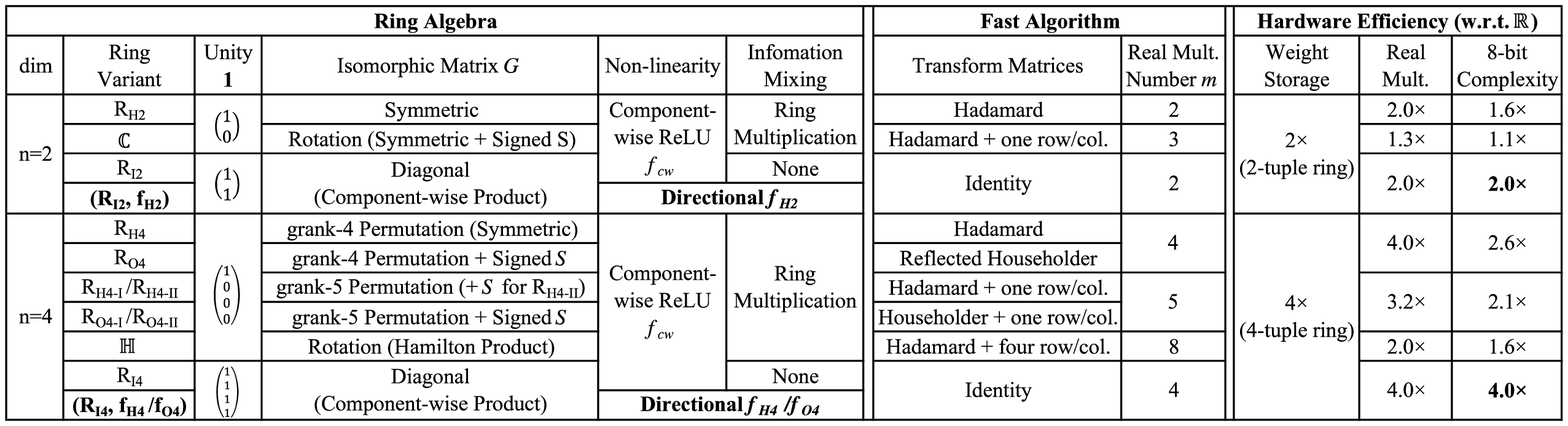}
\label{tab:tab_ring_properties}
\end{table*}

\subsection{Proper Ring Multiplication}
\label{ssec:proper_ring_mult}

In addition to the rings at hand, we would like to search more proper variants for in-depth analysis.
We make three practical assumptions to confine the scope of discussion.
The first one is exclusive sub-product distribution: each input sub-product $g_k x_j$ in (\ref{equ:ring_bilinear}) is distributed to one output component $z_i$ exclusively.
It provides complete and non-redundant information mixing between ring components for maintaining compact model capacity.
Then $G$ has full rank and can be formulated by a sign matrix $S$ and a permutation indexing matrix $P$:
\begin{align}
G_{ij} = S_{ij} g_{{\scriptscriptstyle P_{ij}}}, \label{equ:g_simple}
\end{align}
where $S_{ij}\in \{1, -1\}$, and each row or column of $P$ is a permutation of $\{0,1,...,n-1\}$.
For example, the rotation matrix
$
\big(\begin{smallmatrix}
g_0 & -g_1 \\
g_1 & g_0
\end{smallmatrix}\big)$
for $\mathbb{C}$ has
$S=
\big(\begin{smallmatrix}
1 & -1 \\
1 & 1
\end{smallmatrix}\big)$
and
$P=
\big(\begin{smallmatrix}
0 & 1 \\
1 & 0
\end{smallmatrix}\big)$.

Furthermore, given the existence of a ring unity $\mathbf{1}$, we consider the following explicit condition:
\begin{align}
G=
\left(\begin{smallmatrix}
g_0  \\
g_1 & g_0\\
\vdots & \mbox{ } & \ddots\\
g_{n-1} & \mbox{ } & \mbox{ } & g_0
\end{smallmatrix}\right)
\mbox{ and }
\mathbf{1}=
\left(\begin{smallmatrix}
1 \\
0 \\
\vdots \\
0
\end{smallmatrix}\right). \tag{C1} \label{equ:c1}
\end{align}
Without loss of generality, the condition on the first column of $G$ and $\mathbf{1}$ is drawn from the permutation definition of $P$ and $g \cdot \mathbf{1} = g$.
The diagonal of $G$ is then derived by $\mathbf{1} \cdot g = g$ which states that the isomorphic matrix of $\mathbf{1}$ is the identity matrix, and therefore $G=g_0I$ if $g=g_0 \mathbf{1}$.

The second assumption is \textit{commutativity}.
It is not necessary for constructing neural networks, e.g. quaternions $\mathbb{H}$ are not commutative.
But it is sufficient to enable the demanded \textit{associativity} for a ring, together with the exclusive sub-product distribution and an additional condition on commutative permutation.
The details are discussed in Appendix \ref{ssec:proof_comm}.
Then, by examining the matrix form $Gx=Xg$ for $g\cdot x = x\cdot g$, we have a cyclic-mapping condition for reducing candidates:
\begin{align}
\mbox{If } P_{ij}=j' \mbox{ , then } P_{ij'}=j \mbox{ and } S_{ij}=S_{ij'}.
\tag{C2} \label{equ:c2}
\end{align}

Finally, the last assumption is that a smaller $\mbox{grank}(M)$ is preferred for saving computations and leads to this rule:
\begin{align}
\mbox{Consider only } S\in \argmin_{S'} \mbox{grank}(M(S';P)).
\tag{C3} \label{equ:c3}
\end{align}
In practice, for each $P$ satisfying (\ref{equ:c1}) and (\ref{equ:c2}) we ran the CP-ARLS algorithm \cite{cp_arls_2017} in MATLAB to evaluate $\mbox{grank}(M(S';P))$ for all possible $S'$ and determined ring variants based on the results.

In the following, we consider moderate sparsity for computational imaging with $n=2 \mbox{ and } 4$.
We searched new ring variants as mentioned above and determined their transform matrices as discussed in Section \ref{ssec:fast_ringmul}.
Our findings are listed in Table \ref{tab:tab_ring_properties} where we distinguish ring symbols by indicating $n$ in the subscripts for clarity.
For $n=2$, only $R_{H2}$ and $\mathbb{C}$ can satisfy .
For $n=4$, we found, by exhaustion, that there are two such non-isomorphic permutations.
After applying (\ref{equ:c3}), the minimum $\mbox{grank}(M)$ of them is found to be 4 and 5.
The grank-4 permutation leads to two ring variants: $R_{H4}$ and $R_{O4}$ which are diagonalized respectively by Hadamard transform $H$ and a reflected Householder matrix $O=2 L_1 (I-2vv^t)$ where $L_1=\mbox{diag}(
\left(\begin{smallmatrix}
1 & -1 & -1 & -1
\end{smallmatrix}\right)^t
)$ and $v=\frac{1}{2}
\left(\begin{smallmatrix}
1 & 1 & 1 & 1
\end{smallmatrix}\right)^t
$.
On the other hand, there are four grank-5 ring variants.
Two of them, $R_{H4\mbox{\scriptsize{-I}}}$ and $R_{H4\mbox{\scriptsize{-II}}}$, have transform matrices related to $H$, and the other two, $R_{O4\mbox{\scriptsize{-I}}}$ and $R_{O4\mbox{\scriptsize{-II}}}$, are similarly connected to $O$.
In particular, $R_{H4\mbox{\scriptsize{-I}}}$ applies circular convolution as CirCNN and needs five real multiplications for complex Fourier transform.
The details of isomorphic $G$ and fast algorithms are summarized in Table \ref{tab:tab_ring_matrix_detail}.

\begin{table}
\centering
\caption{Details of isomorphic $G$ and fast algorithms.}
\includegraphics[width=8.7cm]{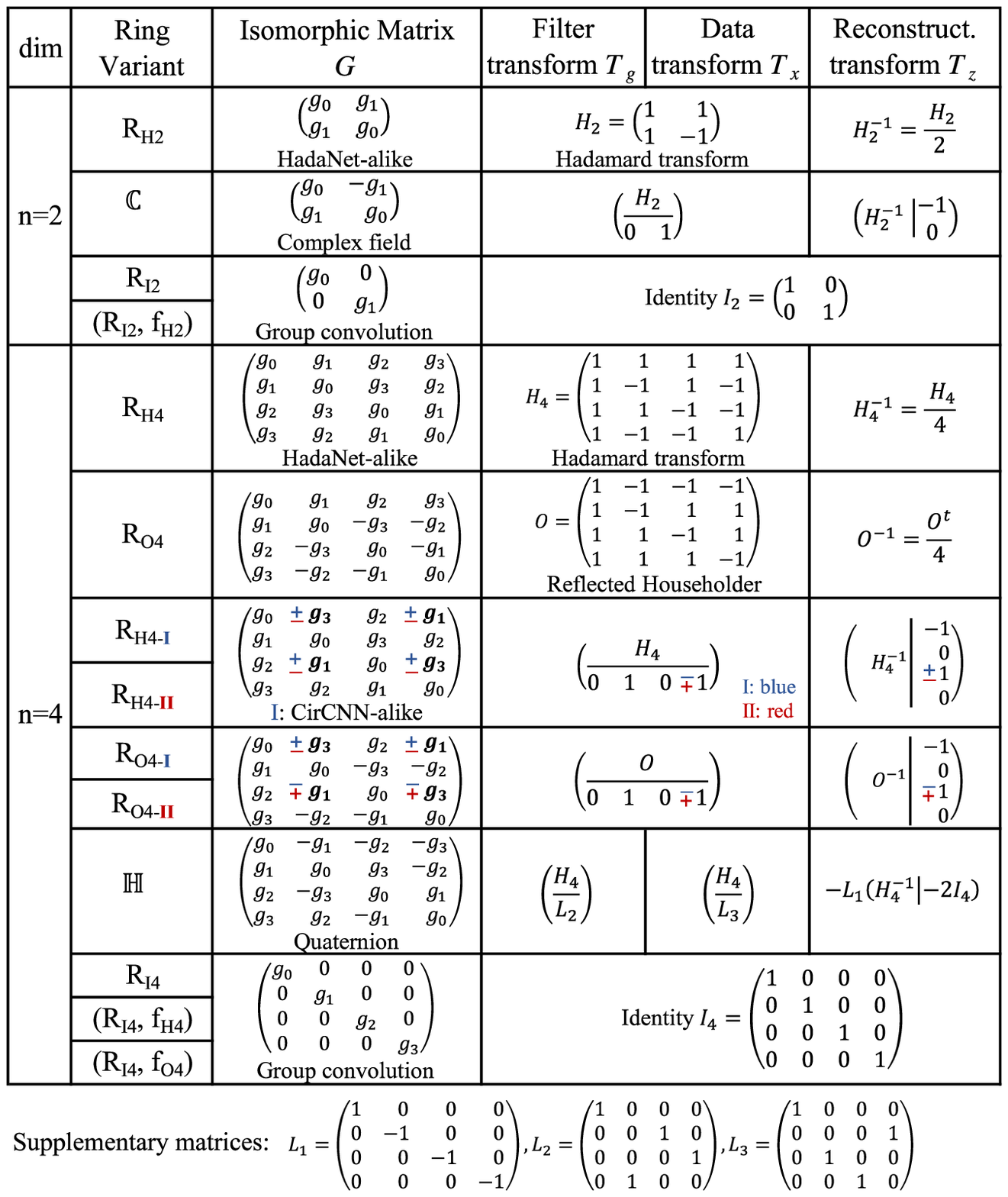}
\label{tab:tab_ring_matrix_detail}
\end{table}

\subsection{Hardware Efficiency}
\label{ssec:computation_property}


Now we can systematically examine the benefits of these rings in terms of hardware resources.
For concise hardware analysis, we assume that different algebraic structures have the same bitwidths for layer inputs and parameters.
Then the weight storage is directly proportional to the degrees of freedom (DoF), and the multiplier complexity can be evaluated on the same basis. 
Regarding the amount of filter weights, a real-valued network would require $n^2$ weights for an $n$-tuple pair of input and output features.
But using $n$-tuple rings instead will only need $n$ real-valued weights to represent the matrix $G$, i.e.\ DoF of $G$ is reduced from $n^2$ to $n$.
Therefore, the efficiency of weight storage with respect to the real-valued networks is $n\times$, e.g.\ $2\times$ and $4\times$ for 2- and 4-tuple rings respectively.
Similarly, the corresponding efficiency in terms of real-valued multiplications can be derived as $n^2/m$.
In Table \ref{tab:tab_ring_properties},  only $R_I$, $R_H$, and $R_{O4}$ can reach the maximum efficiency, $n\times$, for full-rank $G$.

More importantly, for practical implementation we need to consider fixed-point computation and include the bitwidths of the multiplications for precise evaluations.
Fig. \ref{fig:fig_fixedpt_fastalg} shows such an example for the fast algorithm (\ref{equ:ring_fast1})-(\ref{equ:ring_fast3}).
The main overheads brought by the transforms are the increased bitwidths for $\tilde{x}$ and $\tilde{g}$, e.g.\ $T_x$ and $T_g$ will transform $w$-bit $x$ and $g$ into wider $w_x$-bit $\tilde{x}$ and $w_g$-bit $\tilde{g}$.
The circuit complexity of a multiplier can be approximated by the product of its input bitwidths.
We further consider this factor, $w_x \times w_g$,  for evaluating the multiplier complexity for 8-bit features and weights as shown in the rightmost column of Table \ref{tab:tab_ring_properties}.
In this case, only $R_I$ can reach the maximum efficiency for using identity transforms, and the other rings all suffer the corresponding overheads induced by their transforms.
For example, $R_{H4}$ and $R_{O4}$ merely achieve $2.6\times$ efficiency which is $1.6\times$ worse than $R_{I4}$.

\begin{figure}
\centering
\includegraphics[width=7.1cm]{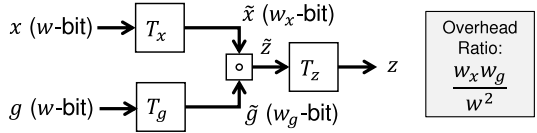} 
\caption{Fixed-point computation of fast algorithms for ring multiplication.
}
\label{fig:fig_fixedpt_fastalg}
\end{figure}

\subsection{Proposed Ring with Directional ReLU}

The above discussions only involve linear operations of neural networks.
For non-linearity, the component-wise ReLU $f_{cw}$ is conventionally adopted even when we actually operate on $n$-tuples.
As a result, $R_I$ will have the worst model capacity, although it has the best hardware efficiency.
It is because the information between different components of an $n$-tuple is not communicated or mixed, which is the same as the discussion on group convolution in \cite{shufflenet_2018}.
This is also the reason why we assumed the complete information mixing property for searching ring multiplications in Section \ref{ssec:proper_ring_mult}.
In the following, we apply algebraic-architectural co-design to have the hardware advantages of $R_I$ while recovering the model capacity.

By examining the fast algorithm (\ref{equ:ring_fast1})-(\ref{equ:ring_fast3}), we found that the information is in fact mixed by the transforms for data, $T_x$, and reconstruction, $T_z$.
In addition, for neural networks this should be required only near non-linearity because cascaded linear operations will simply degrade to another single linear operator.
Based on these two observations, we propose to mix information only before and after non-linearity and thus can adopt $R_I$ for linear operations to have its architectural benefits.
This proposal leads to a novel algebraic function for ring non-linearity: directional ReLU $f_{dir}(y)\triangleq U f_{cw}(Vy)$,
where $U$ and $V$ are two $n\times n$ matrices for an input $n$-tuple $y$.
It is equivalent to performing non-linearity in the directions of the row vectors of $V$, instead of the conventional standard axes, and then turning the axes to the column vectors of $U$.
Thus the components of an $n$-tuple are considered as a whole, not separately, for non-linearity.

The computation of $U$ and $V$ induces complexity overheads.
But they are only linearly proportional to the number of output channels, unlike the bitwidth-increased products in (\ref{equ:ring_fast2}) which grow quadratically.
To further reduce the overhead, we consider the simple Hadamard transform in Table \ref{tab:tab_ring_matrix_detail} and propose a novel ring $(R_I,f_H)$ with the directional ReLU as shown in Fig. \ref{fig:fig_fH}:
\begin{align}
f_{H}(y)\triangleq H f_{cw}(Hy).
 \label{equ:f_H}
\end{align}
For $n=4$, another similar variant $(R_{I4},f_{O4})$ with $f_{O4}(y)\triangleq O f_{cw}(Oy)$ is also possible.
They have the same hardware advantages as $R_I$ and possess better model capacity for additional information mixing.
Note that for constructing neural networks they are different from $R_H$ and $R_{O4}$, especially when skip connections exist or some convolutions are not followed by non-linearity.

\begin{figure}[t]
\centering
\includegraphics[width=7cm]{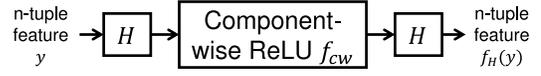} 
\caption{Proposed directional ReLU $f_H$. $H$: $n\times n$ Hadamard transform.
}
\label{fig:fig_fH}
\end{figure}

\section{R\lowercase{ing}CNN Modeling}
\label{sec:RingCNN_construct}

\subsection{Model Construction}

We propose a unified RingCNN framework to include all the considered rings for in-depth comparisons on quality-complexity tradeoffs.
By extending ring algebra to ring tensors $\bold{z}$, $\bold{g}$, and $\bold{x}$, we formulate a $K\times K$ ring convolution (RCONV):
\begin{align}
\bold{z}[p,q,c_o]=\sum_{s,t,c_i} \bold{g}[s,t,c_i,c_o] \cdot \bold{x}[p-s,q-t,c_i],
 \label{equ:rconv_direct}
\end{align}
where $c_o$ and $c_i$ are indexes for output and input $n$-tuple channels, $p$ and $q$ for feature positions, and $s$ and $t$ for weight positions.
Then a real-valued convolution layer, either with non-linearity or not, can be converted into an RCONV layer as shown from Fig. \ref{fig:fig_rconv}(a) to (b).
In this way, we can convert any existing real-valued model structure into an RingCNN alternative.

\begin{figure}
\centering
\includegraphics[width=8.7cm]{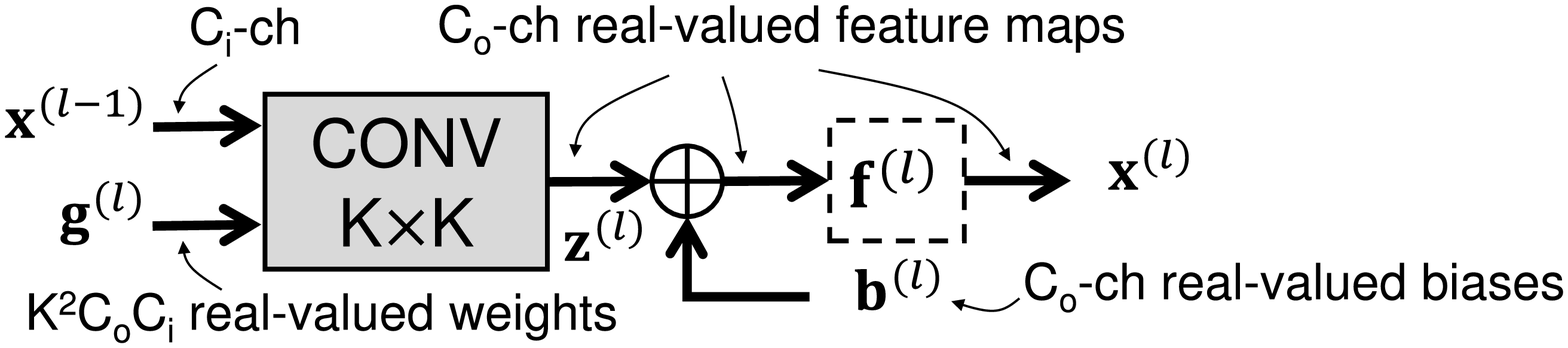} \\
(a) \\
\includegraphics[width=8.7cm]{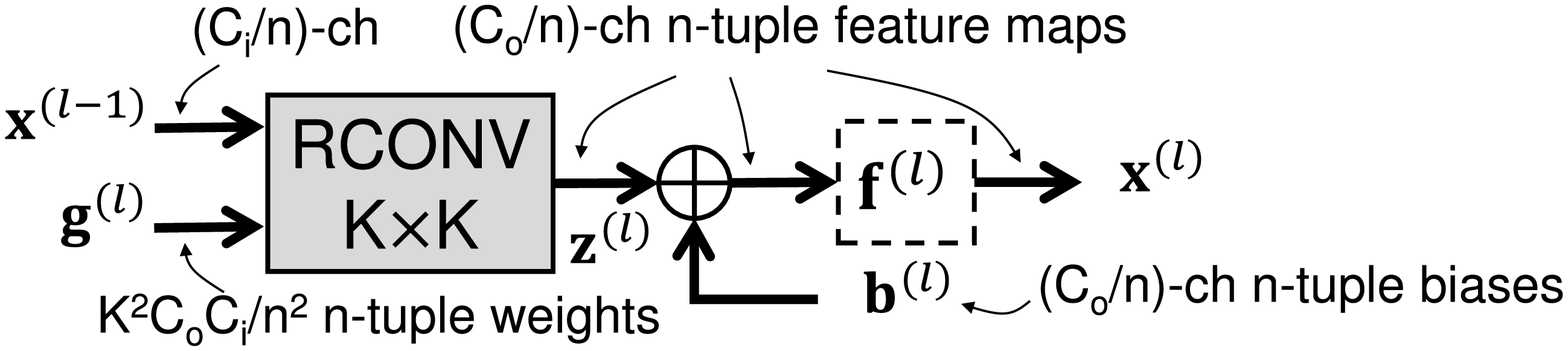} \\
(b) \\
\includegraphics[width=8.7cm]{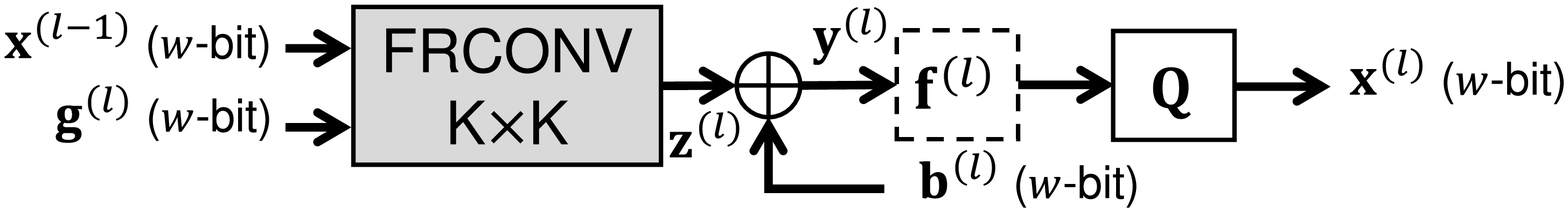} \\
(c)
\caption{$K\times K$ convolution layers with (a) real-valued tensors, (b) $n$-tuple ring tensors, and (c) efficient implementation.
$C_i$, $C_o$: number of real-valued input and output channels.
$\bold{f}$: non-linear tensor operation using element-wise $f$, which could appear or not (dash line) based on model structures. 
$\bold{Q}$: tensor quantization using element-wise quantization $Q$.
}
\label{fig:fig_rconv}
\end{figure}

\subsection{Model Training}

An RingCNN model can be treated as a conventional real-valued CNN if we implement it in form of the matrix-vector multiplication (\ref{equ:ring_iso_mul}).
Then the Backprop algorithm can flow gradients as usual without any special treatment.
For the completeness of ring algebras, we can also represent the gradients in terms of ring operations and then express Backprop using only the ring terminology.
For example, we have $\nabla_x L=G^t \nabla_z L$ from (\ref{equ:ring_iso_mul}) for a training loss $L$.
Then $\nabla_x L=g \cdot \nabla_z L$ for $R_I$, $R_H$, and $R_{O4}$ since $G$ is symmetric for them.
Similarly, the gradient $\nabla_x L$ equals to $g^c \cdot \nabla_z L$ for $R_{H4\mbox{\scriptsize{-I}}}$ and $g^* \cdot \nabla_z L$ for $\mathbb{H}$, where $g^c$ and $g^*$ represent circular folding and quaternion conjugate of $g$ respectively.
The same approach can be applied to express $\nabla_g L$ in ring operations.

\subsection{Efficient Implementation}
\label{ssec:ringcnn_model_opt}

\textbf{Dynamic fixed-point quantization.}
We prefer fixed-point computation for hardware implementation.
It has been shown effective to apply dynamic quantization with separate per-layer Q-formats \cite{qformat_arm_2001} for real-valued feature maps and parameters \cite{eCNN_2019}.
We found that this approach also works well for the RingCNN models that adopt the component-wise ReLU.
But when the directional ReLU is applied, image quality is deteriorated in many cases.
It is because after this non-linearity different ring components have different dynamic ranges, and using one single Q-format for them causes large saturation errors.
Therefore, for the directional ReLU we propose to use component-wise Q-formats for feature maps to address this issue. In other words, there are $n$ different feature Q-formats in one layer, and each component of $n$-tuple features follows its corresponding Q-format.

\textbf{Fast algorithm.}
We use the fast algorithm to formulate a fast ring convolution (FRCONV):
\begin{align}
\bold{z}[p,q,c_o]=T_z \left(  \sum_{s,t,c_i} \tilde{\bold{g}}[s,t,c_i,c_o] \circ \tilde{\bold{x}}[p-s,q-t,c_i] \right) ,
 \label{equ:rconv_fast}
\end{align}
where $\tilde{\bold{g}}$ and $\tilde{\bold{x}}$ are the ring tensors after the transforms $T_g$ and $T_x$ respectively.
For minimizing overheads, we avoid redundant transform operations by applying $T_g$, $T_x$, and $T_z$ only once for each of weight, input, and output ring elements respectively.
Then each RCONV layer can then be efficiently implemented in hardware by applying FRCONV to its fixed-point model as shown in Fig. \ref{fig:fig_rconv}(c).
Note that for $R_I$ FRCONV is the same as RCONV for its identity transform matrices.

\section{\lowercase{e}R\lowercase{ing}CNN Accelerator}
\label{sec:RingO}

To show the efficiency on practical applications, we further design an RingCNN accelerator, named eRingCNN, over the proposed ring $(R_I,f_H)$.
For supporting high-throughput computational imaging, we use the highly-parallel eCNN as a backbone architecture and simply replace its real-valued convolution engine by a corresponding one for RCONV.
This portability of linear operations is an advantage of algebraic sparsity, but we need a new and specific design for the directional ReLU.
We implement two sparsity settings for $n=2$ and $4$, and the details are introduced in the following.

\textbf{System diagram.}
Fig. \ref{fig:fig_eringcnn_system} presents the overall architecture.
In one cycle, it can compute (32/$n$)-channel $n$-tuple output features from (32/$n$)-channel $n$-tuple inputs for 4$\times$2 spatial positions.
For both 3$\times$3 and 1$\times$1 convolution engines, the number of MACs is reduced by 50\% and 75\% for the settings $n=2 \mbox{ and } 4$ respectively.
Similarly, the size of the weight memory can be reduced by the same ratios, e.g.~from 1280~KB in eCNN to 640~KB for $n=2$ and 320~KB for $n=4$.
However, for simplicity the parameter compression in eCNN was not implemented; instead, we increase the size by 1.5$\times$ to 960~KB and 480~KB, respectively, to support large models.
The rest architectural differences from eCNN are mainly on the designs of the RCONV engines and the novel directional ReLU.

\begin{figure}
\centering
\includegraphics[width=6.5cm]{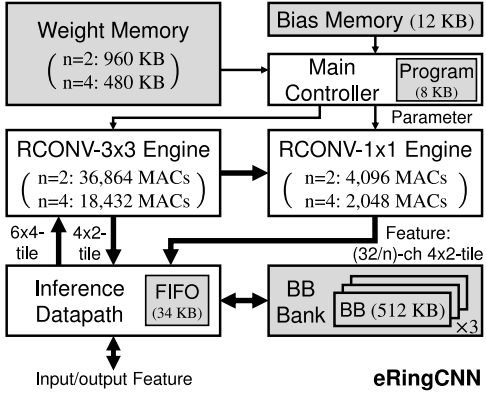}
\caption{System architecture of eRingCNN. (BB: image block buffer)
}
\label{fig:fig_eringcnn_system}
\end{figure}

\textbf{RCONV engine.}
To have local sparsity while maintaining global regularity, we increase the computing granularity from real numbers to $n$-tuple rings.
Fig. \ref{fig:fig_rconv_engine} shows such a modification for the 3$\times$3 convolution engine with $n=4$.
It is a channel-wise 2D computation array for 8-channel 4-tuple inputs and outputs.
Each of the $8\times 8$ computing units is responsible for the 2D 3$\times$3 ring convolution for the corresponding input-output pair with ring tensor weights $\bold{g}_{c_o c_i}\triangleq\bold{g}[\,:\,,\,:\,,c_i,c_o]$.
Thanks to $(R_I,f_H)$, it simply computes component-wise 2D convolutions for saving complexity.
Finally, a novel directional ReLU block, including dynamic quantization with component-wise Q-formats, is devised to replace their real-valued counterparts.

\begin{figure}
\centering
\includegraphics[width=8.7cm]{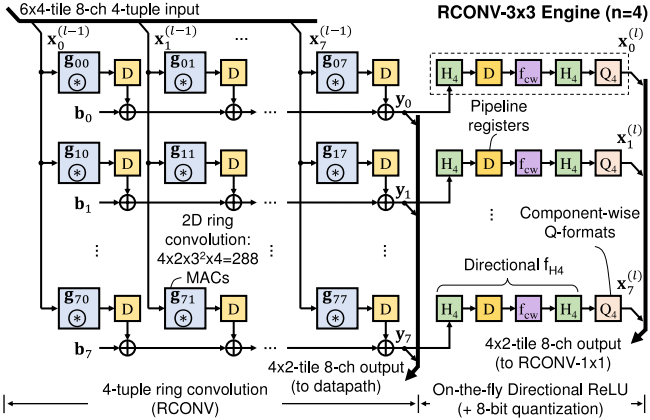}
\caption{RCONV engine for 3$\times$3 filters and 4-tuples. It processes eight 4-tuple input channels and generates eight 4-tuple output channels, which is equivalent to 32-channel real-valued inputs and outputs.
It has 64 computing units (blue boxes) in which $\circledast$ represents a 2D ring convolution unit which generates a 4$\times$2-tile in one cycle, i.e.~$\mathbf{y}_i$ represents the 4-tuples without linearity in 4$\times$2 spatial positions for the $i$-th output channel.
}
\label{fig:fig_rconv_engine}
\end{figure}

\textbf{Directional ReLU unit.}
It mixes information for RCONV outputs to recover model capacity; however, the mixing demands Hadamard transforms on high-bitwidth accumulated outputs, e.g.~24-bit for $n=4$.
This induces two issues for conventional accelerator architectures.
Firstly, the two transforms for $f_H$ are likely to be implemented by the same fixed-point MACs for convolutions to meet the high computation throughput.
But since the weights are only $-1$ and $1$, the hardware efficiency would be low for the multipliers.
Secondly and more importantly, the features before the Hadamard transforms will need to be quantized for the MACs.
We found that these additional quantizations, e.g.~24-bit to 8-bit for $n=4$, would cause up to 0.2~dB of PSNR drop for denoising and SR tasks.

Therefore, we propose an on-the-fly processing pipeline for this novel function, and Fig. \ref{fig:fig_dir_relu} shows our implementation for $n=4$.
It specifically implements the butterfly structures for Hadamard transforms to optimize hardware efficiency and keeps full-precision operations to preserve image quality.
In this case, the internal bitwidths are up to 33-bit, in which the component-wise Q-formats contribute 5-bit for aligning components (through the left-shifters).
This circuit is the major overhead for using $(R_I,f_H)$ and also appears in the inference datapath for the non-linearity after skip or residual connections.

\begin{figure}
\centering
\includegraphics[width=8.7cm]{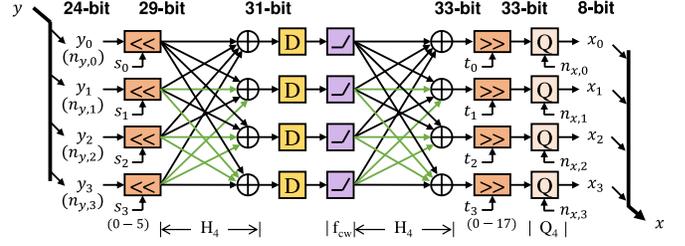}
\caption{On-the-fly directional ReLU for a 4-tuple (dashed box in Fig. \ref{fig:fig_rconv_engine}). The input is 
$y= (y_0, y_1, y_2, y_3)^t$, and the output $x = (x_0, x_1, x_2, x_3)^t$.
With component-wise Q-formats, their numbers of fractional bits are given as $n_{y,i}$ and $n_{x,i}$, respectively.
The numbers of shift bits are then derived as $s_i = \max_{i'} n_{y,i'} - n_{y,i}$ and 
$t_i = \max_{i'} n_{y,i'} - n_{x,i}$.
Green lines represent minus terms of the adders for Hadamard transform.
}
\label{fig:fig_dir_relu}
\end{figure}

\section{Evaluations}
\label{sec:eval}

We show extensive evaluations for (A) ring algebras, (B) image quality on eRingCNN, and (C) hardware performance of eRingCNN.
For clarity, the two sparsity configurations for eRingCNN are denoted by eRingCNN-n2 and eRingCNN-n4.

\subsection{Ring Algebras}

\textbf{Training setting and test datasets.}
For quality evaluations, we use the advanced ERNets for eCNN \cite{eCNN_2019} as the real-valued backbone models.
Then RingCNN models are converted from them as shown from Fig. \ref{fig:fig_rconv}(a) to (b).
To fairly compare RingCNNs and real-valued CNNs, we evaluate their best performance by increasing their initial learning rates as high as possible before training procedures become unstable.
Note that the real-valued ERNets in this paper will therefore perform better than those in \cite{eCNN_2019} because of using higher learning rates.
The models are trained using the lightweight settings as summarized in Table \ref{tab:tab_trainingsetting} if not mentioned.
Finally, we test denoising networks on datasets Set5 \cite{set5}, Set14 \cite{set14}, and CBSD68 \cite{cbsd68}, and super-resolution ones on Set5, Set14, BSD100 \cite{cbsd68}, and Urban100 \cite{urban100}.

\begin{table}
\centering
\caption{Training settings.}
\includegraphics[width=8.2cm]{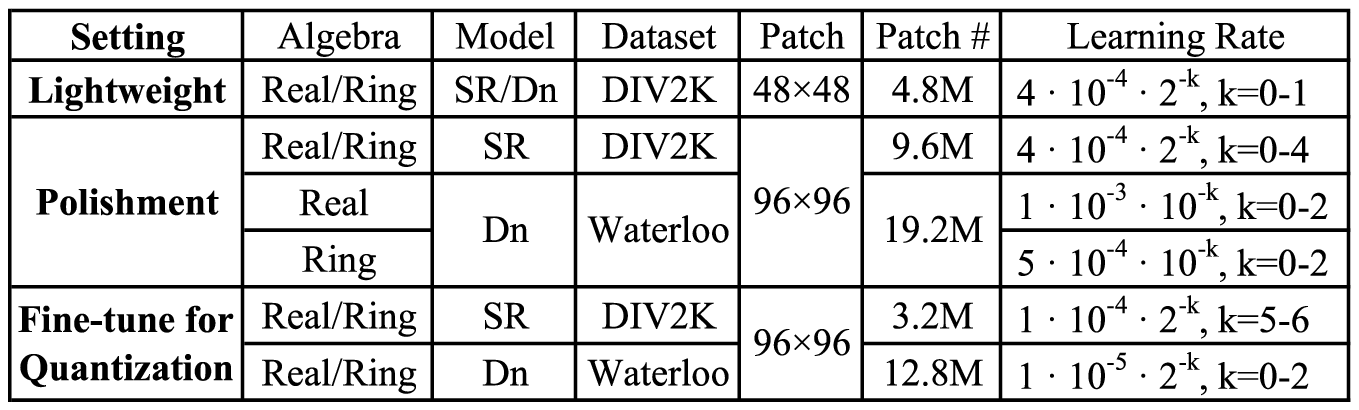}
\label{tab:tab_trainingsetting}
\end{table}

\textbf{Quality comparison for different rings.}
Fig. \ref{fig:fig_psnr_rings} compares image quality in PSNR for the rings in Table \ref{tab:tab_ring_properties}.
When the component-wise ReLU is used, $R_I$ performs the worst due to the lack of information mixing.
The two traditional algebra alternatives $\mathbb{C}$ and $\mathbb{H}$ also do not perform well, considering more real-valued multiplications are required.
Between the two grank-4 variants for $n=4$, the newly-discovered $R_{O4}$ performs better than the HadaNet-alike $R_{H4}$.
Similar results can be found for their corresponding grank-5 variants, e.g.\ the newly-discovered $R_{O4\mbox{\scriptsize{-I}}}$ better than the CirCNN-alike $R_{H4\mbox{\scriptsize{-I}}}$.
However, by using the directional ReLU, the proposed $(R_I, f_H)$ can give better quality and constantly outperform the others.
Since $(R_{I4}, f_{O4})$ shows inferior quality, we therefore focus on $(R_I, f_H)$ and adopt it for our implementation.

\begin{figure}
\centering
\includegraphics[width=8.7cm]{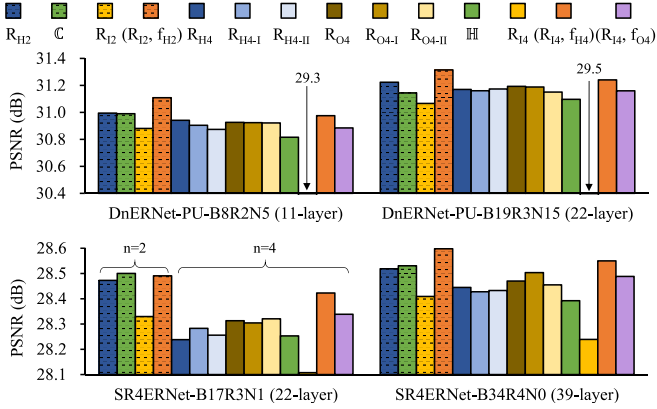} 
\caption{PSNR comparison of different rings with (top) denoising model structure DnERNet-PU (PU: pixel unshuffle) and (bottom) four-times super-resolution (SR$\times 4$) SR4ERNet.
The configurations are the same as the real-valued ERNets: ERModule number $B$, base pumping ratio $R$, and additional pumping layer number $N$.
(Hatch pattern: 2-tuples; solid color: 4-tuples.)
}
\label{fig:fig_psnr_rings}
\end{figure}

\textbf{Ablation study between $(R_I, f_H)$ and $R_H$.}
They share similar structures but have two major differences.
First, $(R_I, f_H)$ multiplies input features by weights $\bold{g}$ directly while $R_H$ does that after applying the filter transform.
Second, $(R_I, f_H)$ applies Hadamard transform only when non-linearity is required, but $R_H$ always does that and results in a redundant structure.
Therefore, $R_H$ can imitate $(R_I, f_H)$ by making up the differences: first training on transformed weights $\tilde{\bold{g}}$ and then modifying model structures accordingly.
Fig. \ref{fig:fig_ablation}(a) shows an example for modifying a residual block, and Fig. \ref{fig:fig_ablation}(b) illustrates typical PSNR results using two SR$\times$4 networks as examples.
Training on $\tilde{\bold{g}}$ is occasionally helpful, but structure modification improves image quality most of the time.
Therefore, the compact model structure is the main reason why $(R_I, f_H)$ outperforms $R_H$ for computational imaging.

\begin{figure}
\centering
\begin{tabular}{>{\centering\arraybackslash}m{3.6cm} >{\centering\arraybackslash}m{5cm}}
\includegraphics[width=4.3cm]{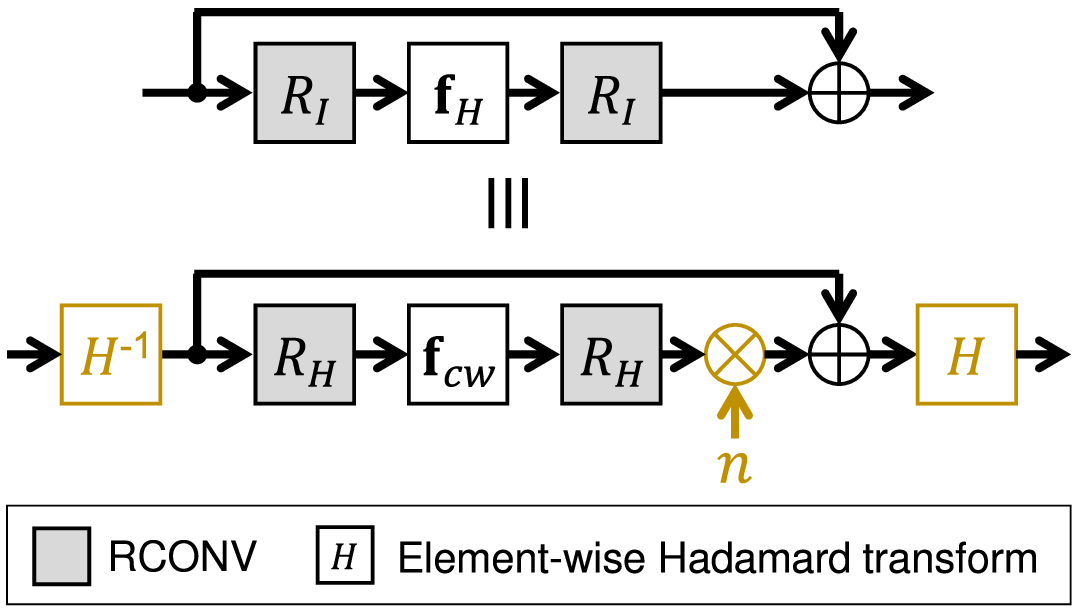} &
\includegraphics[width=4.5cm]{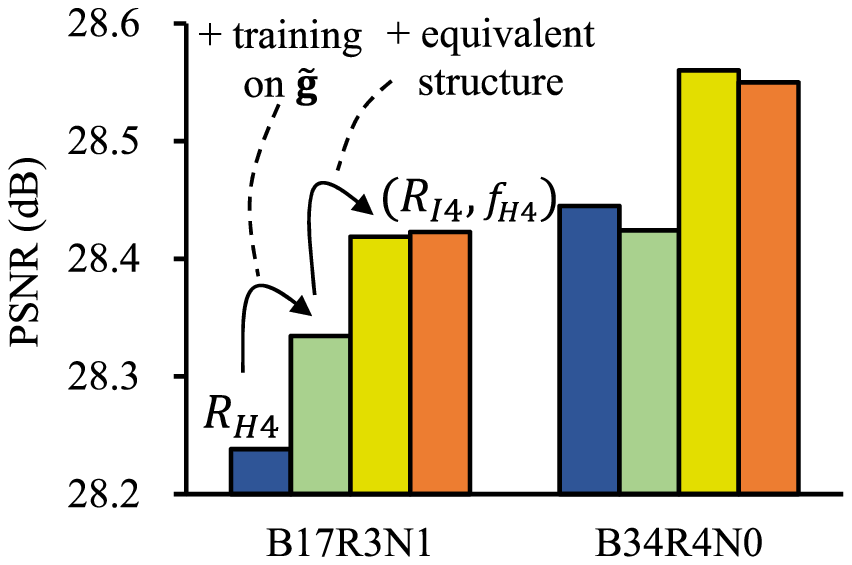} \\
$\qquad$(a) & (b)
\end{tabular}
\caption{Ablation study for $(R_I, f_H)$.
(a) Equivalent residual block structures for using (top) $(R_I, f_H)$ and (bottom) $R_H$.
(b) PSNR for two SR4ERNet model structures ($n$=4).
}
\label{fig:fig_ablation}
\end{figure}

\textbf{Comparison with weight pruning.}
We also compare image quality between the proposed algebraic sparsity and the unstructured magnitude-based weight pruning.
While RingCNNs are trained directly, real-valued CNNs are first pre-trained, then pruned, and finally fine-tuned.
Fig. \ref{fig:fig_ring_vs_pruning} shows the comparison results, and RingCNNs over $(R_I, f_H)$ can deliver better image quality than the weight pruing for compression ratios 2$\times$, 4$\times$, and 8$\times$.
In particular, the 2-tuple networks can even outperform the original (1$\times$) real-valued networks in many cases. This shows that the algebraic sparsity can serve strong prior for CNN models.
As a result, $(R_I, f_H)$ not only provides more regular structures but also achieves better quality for computational imaging.
A case for recognition tasks, though not the focus of this paper, is also studied in Appendix \ref{ssec:vs_legr}, where convolutions and corresponding non-linear functions are implemented with $(R_I, f_H)$, and batch normalization is remained as real-valued operations.

\begin{figure}
\centering
\begin{tabular}{>{\centering\arraybackslash}m{1.1cm} >{\centering\arraybackslash}m{7.3cm}}
\includegraphics[width=1.5cm]{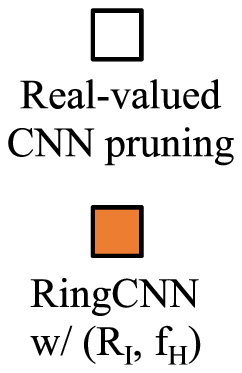} &

\begin{tabular}{c}
\includegraphics[width=6.8cm]{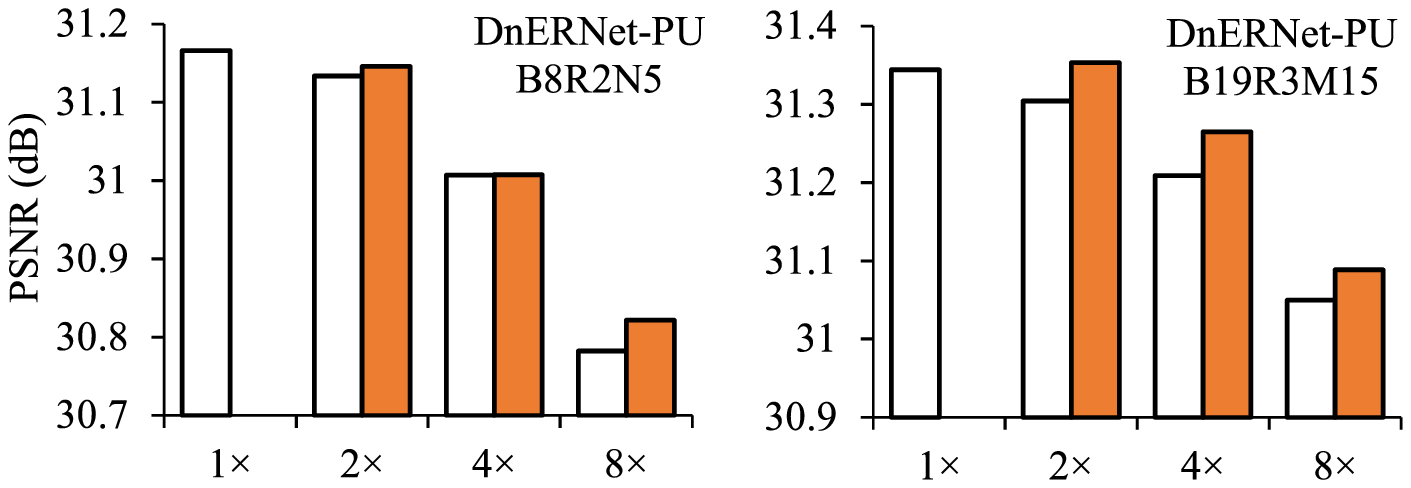} \\
\includegraphics[width=6.8cm]{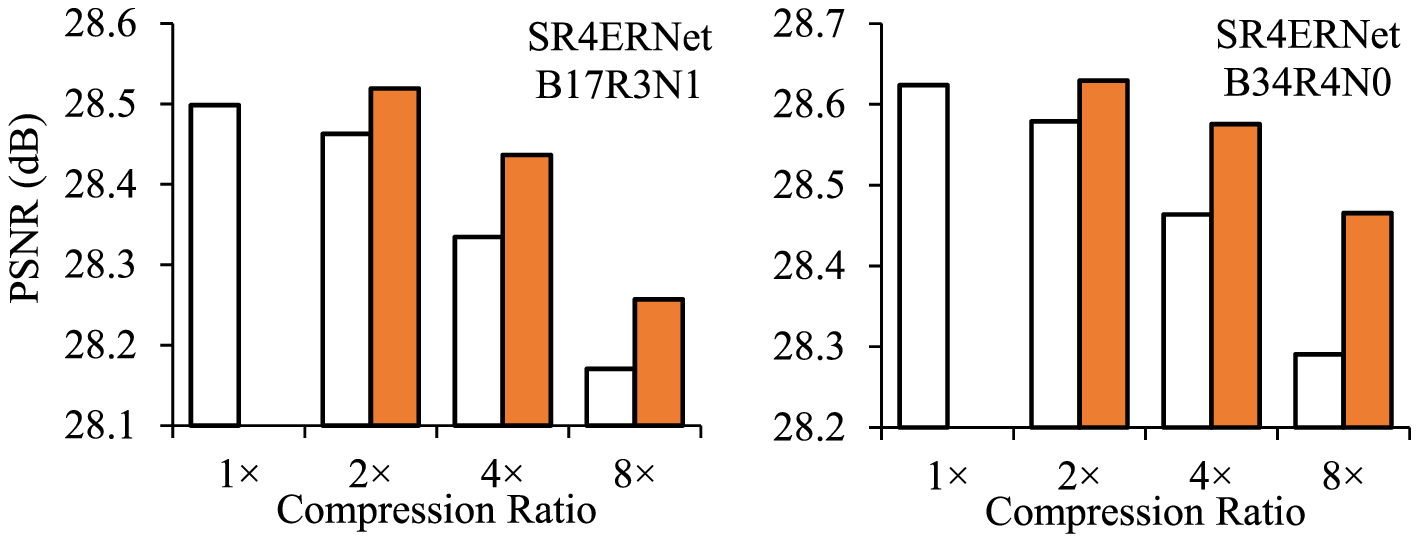} \\
\end{tabular}

\end{tabular}
\caption{Algebraically-sparse RingCNN versus unstructured weight pruning.
$(R_I, f_H)$: $n$=2, 4, and 8 for 2$\times$, 4$\times$, and 8$\times$ compressions.
We use 200 more epochs for fine-tuning the weight-pruned models and, for fair comparisons, 100 more epochs for the original (1$\times$) CNN and RingCNNs.
}
\label{fig:fig_ring_vs_pruning}
\end{figure}

\textbf{Fixed-point implementation.}
For comparisons in hardware efficiency, we implemented highly-parallel FRCONV engines, as depicted in Fig. \ref{fig:fig_rconv}(c) with non-linearity, for different rings.
Their RTL codes are synthesized with 40~nm CMOS technology.
For quality comparison, the models are quantized in 8-bit and then fine-tuned using the setting at the bottom of Table \ref{tab:tab_trainingsetting}.
The quality loss due to quantization is similar for each ring, and Fig. \ref{fig:fig_gate_vs_quality} shows the comparison results.
The area efficiencies are very close to the estimated 8-bit complexity in Table \ref{tab:tab_ring_properties} because convolutions dominate the areas.
The proposed $(R_I, f_H)$ can provide the smallest area and the best quality at the same time.
Compared to the CirCNN-alike $R_{H4\mbox{\scriptsize{-I}}}$ and HadaNet-alike $R_{H4}$, it has nearly 0.1~dB PSNR gain for the SR$\times$4 task and provides 1.8$\times$ and 1.5$\times$ area efficiencies respectively.
In summary, $R_I$ can save area efficiently, and $f_H$ can recover image quality significantly.

\begin{figure}
\centering
\includegraphics[width=7cm]{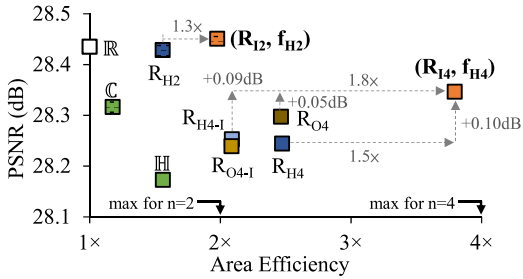} 
\caption{Area efficiency after logic synthesis versus PSNR for 8-bit fixed-point implementation.
These engines can compute one 3$\times$3 convolution layer in one cycle for 32 input and 32 output channels of 8-bit real-valued features.
Area efficiencies are calculated with respect to the real-valued engine, and PSNR is evaluated on SR4ERNet-B17R3N1.
}
\label{fig:fig_gate_vs_quality}
\end{figure}

\subsection{Image Quality on eRingCNN}

\textbf{Training setting and model selection.}
To show competitive image quality, we further train models using the polishment setting in Table \ref{tab:tab_trainingsetting} with two large datasets, DIV2K \cite{div2k} and Waterloo Exploration \cite{waterloo}.
We consider two throughput targets for hardware acceleration: \textit{HD30} for Full HD 30~fps and \textit{UHD30} for 4K Ultra-HD 30~fps.
For each throughput target and application scenario, we adopt the compact ERNet configuration for the real-valued eCNN in \cite{eCNN_2019}.
It has been optimized over model depth and width in terms of PSNR, and we build its corresponding RingCNN models with $(R_I, f_H)$.

\textbf{Floating-point models.}
The PSNR results are shown in Table \ref{tab:tab_psnr_floating}.
The RingCNN models show significant gains over the traditional CBM3D \cite{bm3d_2007} for denoising and VDSR for SR$\times$4.
Compared to the advanced FFDNet and SRResNet, the models for eRingCNN-n2 can outperform them with PSNR gains up to 0.15~dB at HD30 and have similar quality at UHD30.
With 75\% sparsity, the models for eRingCNN-n4 still give superior quality and only show noticeable PSNR inferiority for denoising at UHD30 due to the shallow layers.

\begin{table}
\centering
\caption{PSNR performance and comparison for models on eRingCNN.}
\includegraphics[width=8.7cm]{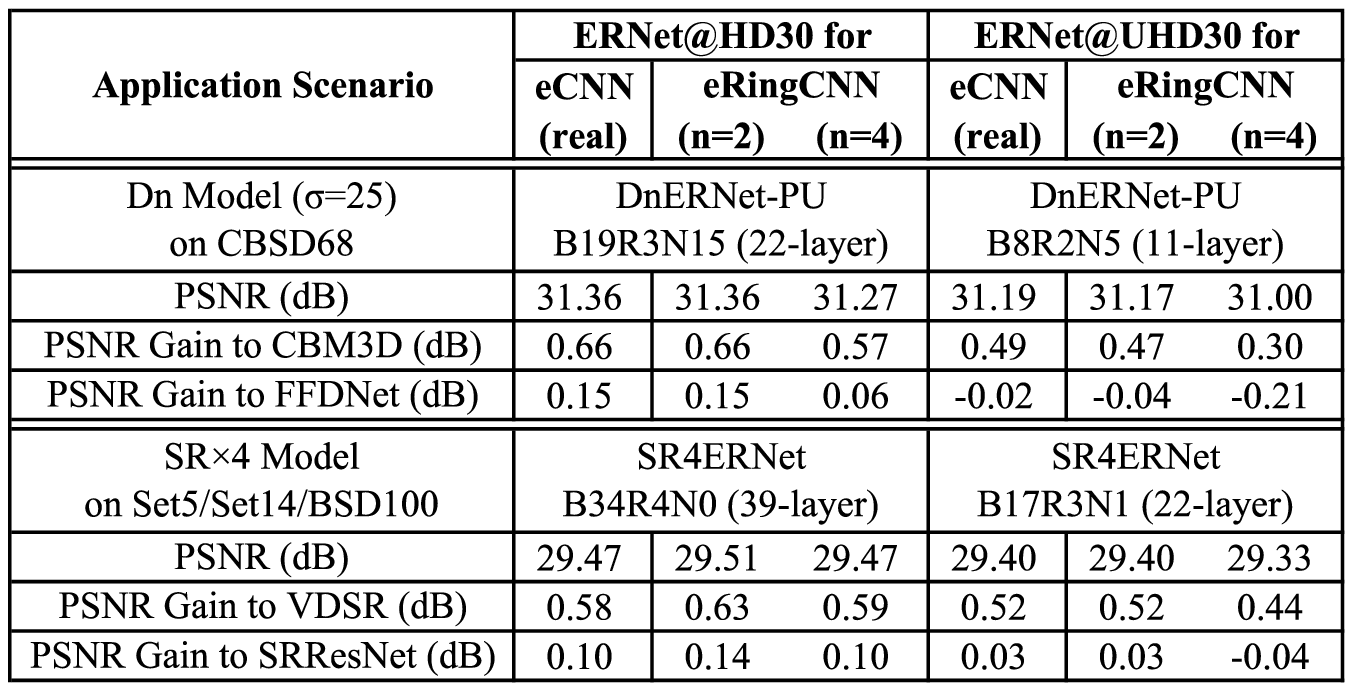}
\label{tab:tab_psnr_floating}
\end{table}

\textbf{Dynamic fixed-point quantization.}
On the top of Fig. \ref{fig:fig_avg_psnr}, we show the effect of the 8-bit dynamic quantization which is used to save area.
The quality degradation for ring tensors is around 0.11-0.12~dB of average PSNR drops, which is similar to the case of using real numbers.
We also show the effect of applying sparse ring algebras (eCNN$\Rightarrow$eRingCNN) at the bottom of Fig. \ref{fig:fig_avg_psnr}. 
The degradation is not obvious for $n=2$, and the models for eRingCNN-n2 even outperform those for eCNN by 0.01~dB on average.
For $n=4$, those for eRingCNN-n4 only suffer small PSNR degradation in 0.11~dB.

\begin{figure}
\centering
\includegraphics[width=8.7cm]{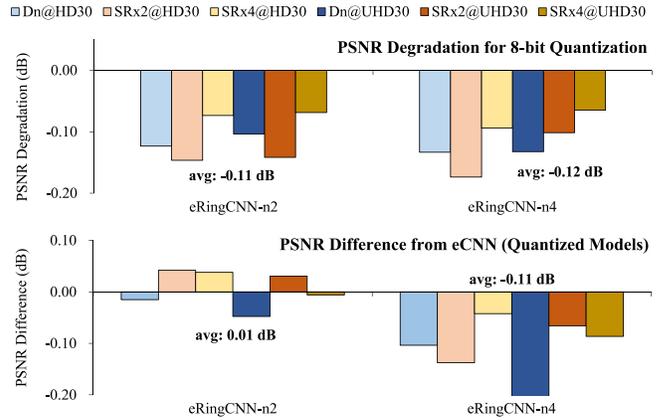}
\caption{Average PSNR results for six application targets:
(top) PSNR degradation of 8-bit quantized ERNet models from float-point ones and
(bottom) PSNR differences of quantized models for eRingCNN from those for eCNN.
The test datasets for denoising (Dn) are Set5, Set14, and CBSD68, and those for super-resolution (SR) are Set5, Set14, BSD100, and Urban100.
}
\label{fig:fig_avg_psnr}
\end{figure}

\subsection{Hardware Performance of eRingCNN}

\textbf{Implementation and CAD tools.}
We developed RTL codes in Verilog and verified the functional validity based on bit- and cycle-accurate simulations.
Then the verified RTL codes were synthesized using Synopsys Design Compiler with TSMC 40~nm CMOS library.
And SRAM macros were generated by ARM memory compilers.
We used Synopsys IC Compiler for placement and routing and generating layouts for five well-pipelined macro circuits which constitute eRingCNN collectively.
Finally, we performed time-based power consumption using Synopsys Prime-Time PX based on post-layout parasitics and dynamic signal activity waveforms from RTL simulation.

\textbf{Hardware performance.}
We show the design configurations and layout performance in Table \ref{tab:tab_eringcnn_spec}.
The areas of eRingCNN-n2 and eRingCNN-n4 are 33.73 and 23.36 $\mbox{mm}^2$ respectively, and the corresponding power consumptions are 3.76 and 2.22~W.
They mainly differ in the ring dimension, and eRingCNN-n4 uses only a half number of MACs and a half size of weight memory compared to eRingCNN-n2.

\begin{table}
\centering
\caption{Design configurations and layout performance of eRingCNN.}
\includegraphics[width=6cm]{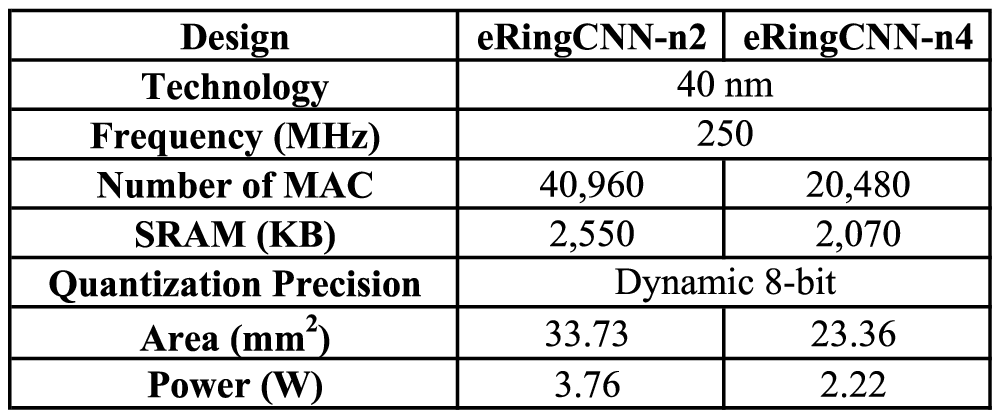}
\label{tab:tab_eringcnn_spec}
\end{table}

\textbf{Area and power breakdown.}
The details are shown in Table \ref{tab:tab_eringcnn_breakdown}.
For eRingCNN-n2, the convolution engines contribute 57.42\% of area and 86.51\% of power consumption for the highly-parallel computation.
And for eRingCNN-n4 their contributions go down to 45.63\% and 76.56\%, respectively, because of the saving of MACs.
In addition, for a larger $n$ the directional ReLU uses more adders and causes wider bitwidths.
Therefore, for the RCONV-3$\times$3 engines it occupies only 3.4\% of area for eRingCNN-n2 but up to 8.9\% for eRingCNN-n4.
Accordingly, the inference datapath in eRingCNN-n4 is also 0.53 $\mbox{mm}^2$ larger than that in eRingCNN-n2.


\begin{table}
\centering
\caption{Area and power breakdowns.}
\includegraphics[width=8.7cm]{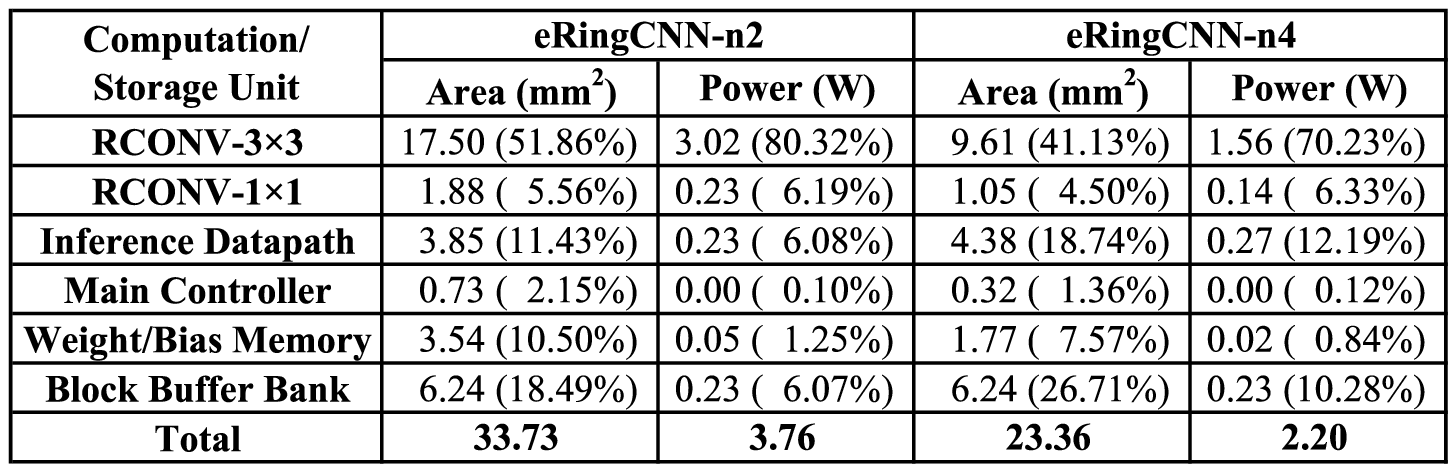}
\label{tab:tab_eringcnn_breakdown}
\end{table}

\textbf{Comparison with eCNN.}
As shown in Fig. \ref{fig:fig_comp_ecnn}, the RCONV engines reach near-maximum hardware efficiencies $(\cong n)$.
Those in eRingCNN-n2 achieve 2.08$\times$ area efficiency and 2.00$\times$ energy efficiency.
And those in eRingCNN-n4 further increase the efficiency gains of area and energy to 3.77$\times$ and 3.84$\times$.
The numbers for the whole accelerator are as high as 1.64$\times$ and 1.85$\times$ for eRingCNN-n2, and 2.36$\times$ and 3.12$\times$ for eRingCNN-n4.
In addition, we compare their quality-energy tradeoff curves in Fig. \ref{fig:fig_quality_vs_energy}.
The eRingCNN accelerators show clear advantages over eCNN; in particular, the low-complexity eRingCNN-n4 is preferred when less energy is allowed to be consumed for generating one pixel.
Finally, as eCNN, the eRingCNN accelerators demand only 1.93~GB/s of DRAM bandwidth for high-quality 4K UHD applications.

\begin{figure}[t]
\centering
\includegraphics[width=8.7cm]{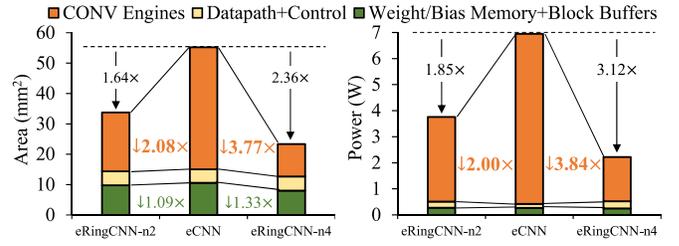}
\caption{Area (left) and power (right) comparison with eCNN.
}
\label{fig:fig_comp_ecnn}
\end{figure}

\begin{figure}[t]
\centering
\includegraphics[width=8.7cm]{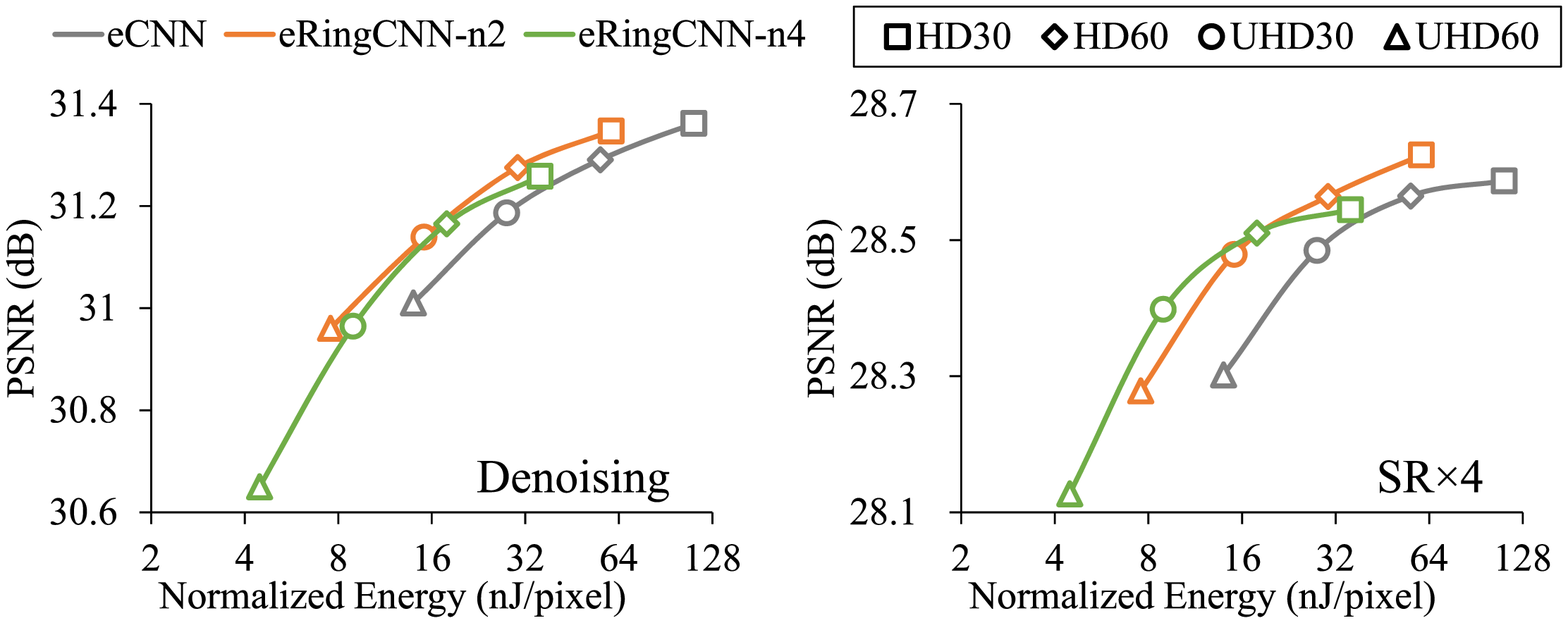}
\caption{Quality-energy comparison with eCNN: (left) denoising and (right) SR$\times$4. The energy is normalized for generating one image pixel.
Each accelerator forms its own curve with compact model configurations over different throughput targets.
}
\label{fig:fig_quality_vs_energy}
\end{figure}

\textbf{Comparison with Diffy.}
Table \ref{tab:tab_comp_ecnn_diffy} compares the hardware performance of Diffy\footnote{We project the silicon area and power consumption of Diffy under 40~nm technology based on the scaling comparison of 65~nm in \cite{tsmc_40nm}: 2.35$\times$ gate density and 0.5$\times$ power consumption under the same operation speed.} \cite{diffy_2018}, another state-of-the-art accelerator for computational imaging, along with eCNN.
Diffy applies optimization on bit-level computation which is hard to compare with eRingCNN directly; therefore, we perform comparison based on the same application target: FFDNet-level inference at Full-HD 20fps.
In this case, the energy efficiencies of eRingCNN-n2 and eRingCNN-n4 over Diffy are 2.71$\times$ and 4.59$\times$ respectively by running at 167~MHz.

\begin{table}[t]
\centering
\caption{Comparisons of eCNN and Diffy for computational imaging.}
\includegraphics[width=8.7cm]{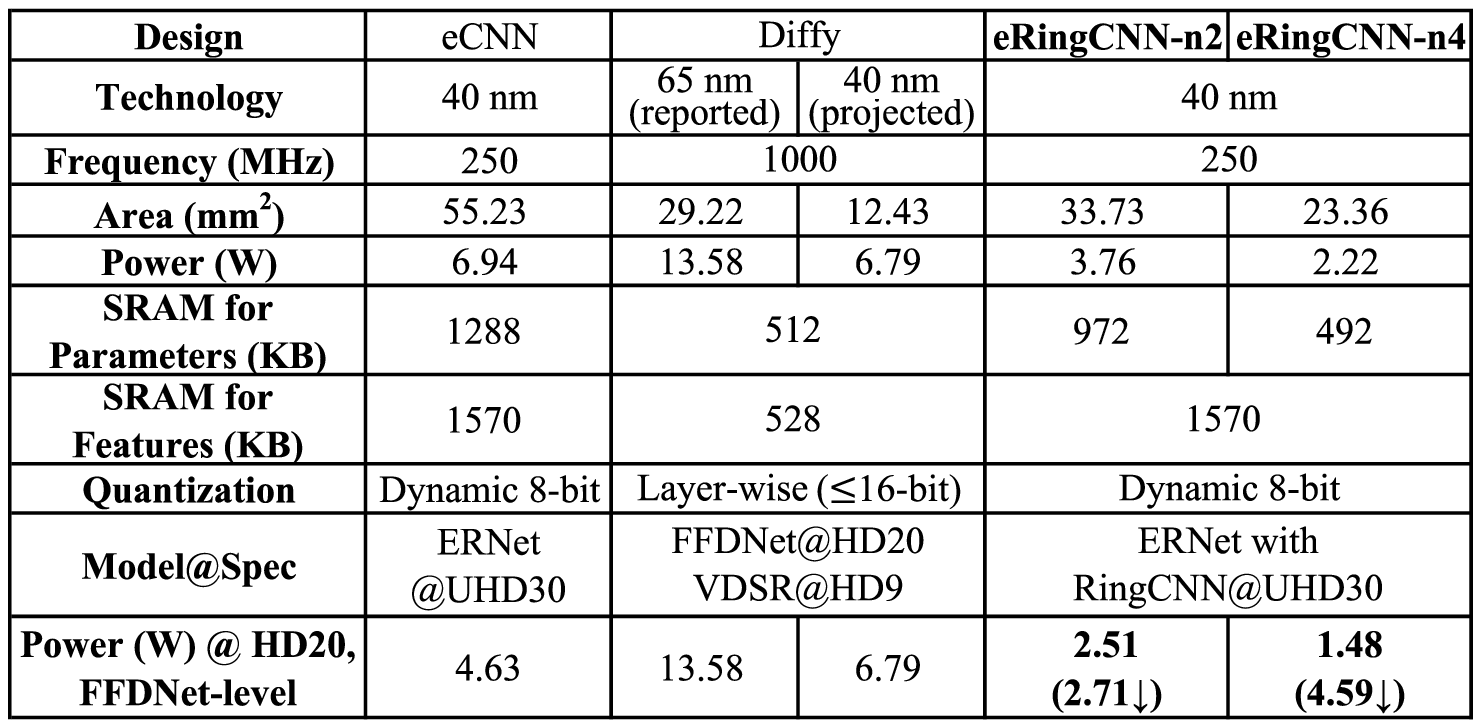}
\label{tab:tab_comp_ecnn_diffy}
\end{table}

\textbf{Comparison with SparTen, TIE, and CirCNN.}
Table \ref{tab:tab_comp_circnn} compares with the state-of-the-arts for different sparsity approaches: SparTen (natural) \cite{sparten_2019}, TIE (low-rank) \cite{tie_2019}, and CirCNN (full-rank).
Here we compare synthesis results because only such numbers are reported for SparTen and CirCNN.
For comparing over different compression ratios, we consider an \textit{equivalent} throughput which corresponds to the computing demand of the target uncompressed or real-valued model.
With only 2-4$\times$ compression, our eRingCNN accelerators already provide competitive energy efficiencies as equivalent 19.1-28.4~TOPS/W. 
In contrast, SparTen merely achieves 2.7~TOPS/W due to significant overheads for handling irregularity.
Although TIE is very efficient for highly-compressed fully-connected (FC) layers, it shows inefficiency for the CONV layers with lower compression ratios.
Finally, CirCNN only provides 10.0~TOPS/W using as high as 66$\times$ compression.
The potential of algebraic sparsity is therefore demonstrated, in particular on moderately-compressed CNNs for computational imaging.

\begin{table}[t]
\centering
\caption{Comparison of SparTen, TIE, and CirCNN (synthesis results).}
\includegraphics[width=8.7cm]{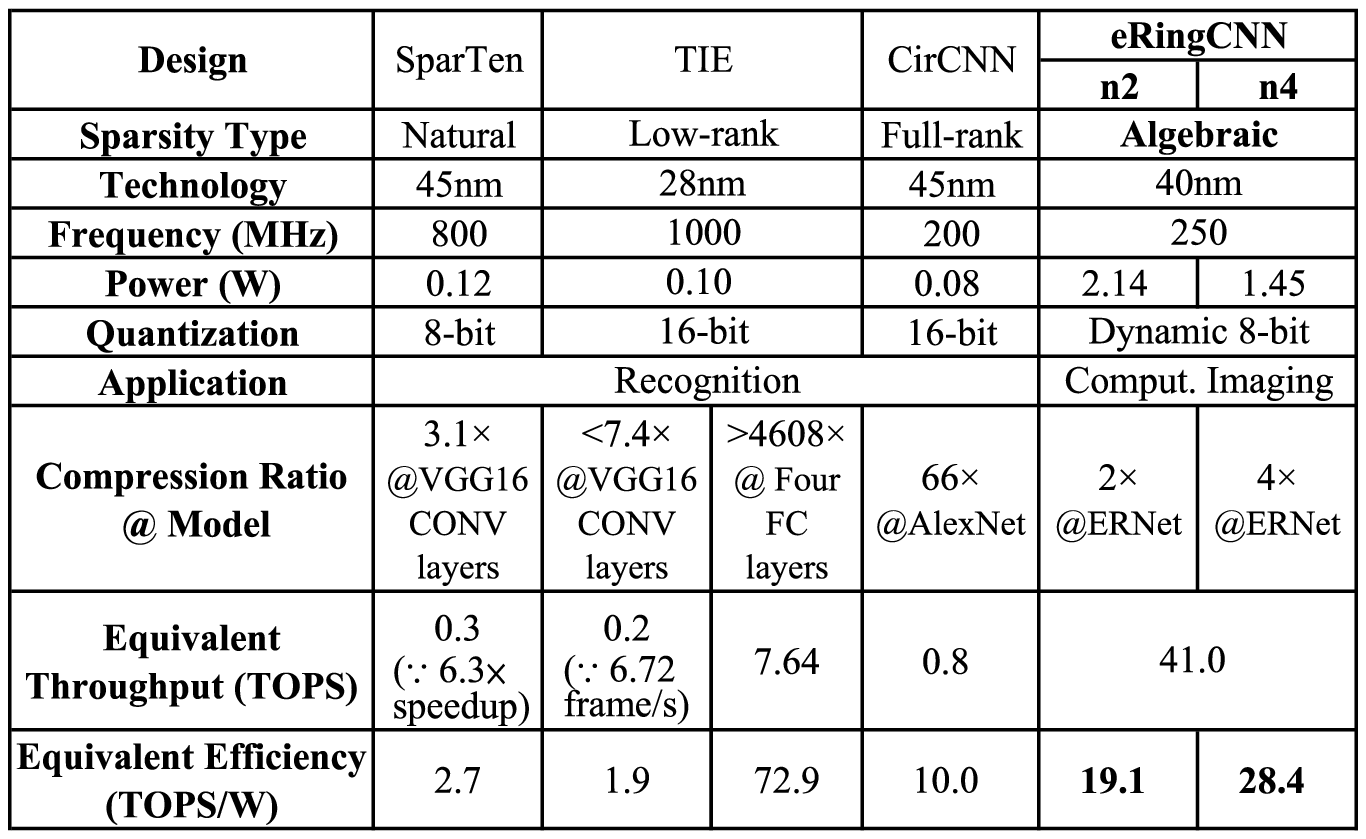}
\label{tab:tab_comp_circnn}
\end{table}

\section{Related Work}
\label{sec:related}

\textbf{Low-rank sparsity.}
This line of research also provides regular structure for efficient hardware acceleration.
One approach is low-rank approximation, including tensor-train (TT) \cite{tt_lowrank_2011}, canonical polyadic (CP) \cite{cp_lowrank_2015}, and Tucker \cite{tucker_lowrank_2016}.
Another one is building networks using low-rank structures, such as MobileNet (v1/v2) \cite{mobilenet_2017,mobilenetv2_2018} and SqueezeNet \cite{squeezenet_2017}.
However, this sparsity mainly aims to provide high compression ratios, and the effect for computational imaging need further exploration.

\textbf{Natural sparsity.}
Exploiting such sparsity in features has been well-studied, e.g.\ in Cnvlutin \cite{cnvlutin_2016}, Cambricon-{X} \cite{cambriconX_2016}, and Diffy \cite{diffy_2018}.
And the sparsity in filter weights is usually explored by pruning \cite{deepcompression_2016,sparsegran_2017}.
They can be further combined as in SCNN \cite{scnn_2017}, SparTen \cite{sparten_2019}, and SmartExchange \cite{smartexchange_2020}.
However, the natural irregularity could lead to high overheads for indexing circuits and load imbalance.


\textbf{Dense CNN accelerators.}
There are many accelerators proposed for general-purpose inference with high parallelism, such as TPU \cite{tpu_2017}, DNPU \cite{dnpu_2018}, ShiDianNao \cite{shidiannao_2015}, and Eyeriss (v1/v2) \cite{eyeriss_2016,eyerissv2_2019}.
However, the computation sparsity for computational imaging were seldom exploited.

\textbf{Block-based inference flows.}
They eliminate huge external memory bandwidth for feature maps, and two approaches were proposed to handle the boundary features across neighboring blocks.
One is feature reusing, such as fused-layer \cite{fusedlayer_2016} and {Shortcut Mining} \cite{scm_2019}, and the other one is recomputing, like eCNN \cite{eCNN_2019}.
In this paper, we adopt the latter only for the purpose of implementation and comparison.

\section{Conclusion}
\label{sec:conclusion}

This paper investigates the fundamental but seldom-explored algebraic sparsity for accelerating computational imaging CNNs.
It can provide local sparsity and global regularity at the same time for energy-efficient inference.
We lay down the general RingCNN framework by defining proper ring algebras and constructing corresponding CNN models.
By extensive comparisons with several rings, the proposed one with the directional ReLU achieves near-maximum hardware efficiency and the best image quality simultaneously.
We also design two high-performance eRingCNN accelerators for verifying practical effectiveness.
They can provide high-quality computational imaging at up to 4K UHD 30~fps while consuming only 3.76~W and 2.22~W, respectively.
Based on these results, we believe that RingCNN exhibits great potentials for enabling next-generation cameras and displays with energy-efficient and high-performance computational imaging.

\section*{Acknowledgment}
The author would like to thank the anonymous reviewers for their feedback and suggestions.
I also would like to thank Chi-Wen Weng for his help on layout implementation.


\bibliographystyle{IEEEtranS}
\bibliography{ref_CNN_20210205}

\renewcommand\thesubsection{\Alph{subsection}}
\renewcommand{\thefigure}{\Alph{subsection}-\arabic{figure}}
\renewcommand{\theequation}{\Alph{subsection}-\arabic{equation}}
\setcounter{figure}{0} 
\setcounter{equation}{0} 

\newtheorem{theorem}{Theorem}[subsection]
\newtheorem{lemma}[theorem]{Lemma}
\newtheorem{corollary}{Corollary}[theorem]

\section*{Appendices}
\label{sec:appendix}

\subsection{Minimal Algorithm for Diagonalizable $G$ over $\mathbb{R}$}
\label{ssec:decomp_g}

\begin{theorem}
Let $m$ be the number of real-valued multiplications in a fast algorithm for an isotropic matrix $G$. Then
\begin{enumerate}[(a)] 
\item The lower bound of $m$ is $\mbox{rank}(G)$.
\item If $G$ is diagonalizable over $\mathbb{R}$, there exists a minimal algorithm such that $m=\mbox{rank}(G)$.
\end{enumerate}
\end{theorem}

\begin{proof}
(a) $m\geq\mbox{rank}(T_z)$ by the dimension theorem and (\ref{equ:ring_fast3}), and $\mbox{rank}(T_z)=\mbox{rank}(G)$ by (\ref{equ:ring_fast3}) and (\ref{equ:ring_iso_mul}). Therefore, we have $m\geq\mbox{rank}(G)$.

(b) Suppose $G=T^{-1} D T$ with an invertible matrix $T$ over real field and a diagonal matrix $D$ which has entries with indeterminates $g_j$.
We have
\begin{align}
z = Gx = \underbrace{T^{-1}}_{=T_z} \underbrace{D}_{=\tilde{g} \circ} \underbrace{T}_{=T_x} x,
 \label{equ:diagonalize}
\end{align}
and $T_g$ can be derived by examining $\tilde{g}_i = \sum_j (T_g)_{ij} g_j = D_{ii}$ for $i=0,...,m-1$.
This ring multiplication achieves the minimum complexity for $m=\mbox{rank}(D)=\mbox{rank}(G)$.
\end{proof}

\subsection{Associativity from Commutativity}
\label{ssec:proof_comm}

\begin{lemma}
\label{as_equi}
Let $R$ be a ring with a bilinear-form multiplication, and $a,b,c\in R$ with corresponding isomorphic matrices $A$, $B$, and $C$.
The the multiplication associativity of $R$ is equivalent to 
\begin{align}
C=AB \mbox{ if } c=a\cdot b, \forall a,b. \notag
\end{align}
\end{lemma}

\begin{proof}
Given $c=a\cdot b$, it is clear because $\forall h\in R, a\cdot (b\cdot h)=(a\cdot b)\cdot h$ $\iff$ $\forall h\in R, ABh=Ch$ $\iff$ $C=AB$.
\end{proof}

\begin{lemma}
\label{permu_equ}
Let $R$ follow the exclusive sub-project distribution. Then
\begin{enumerate}[(a)] 
\item $\forall a \in R$, its isomorphic matrix $A$ can be formulated by $A=\sum_k a_k E_k$ where $E_k$ is a signed permutation matrix.
\item For each standard-basis vector $e_k=I_{:,k}$, $E_k$ is its isomorphic matrix: $e_k \cdot x = E_k x, \forall x \in R$. 
\end{enumerate}
\end{lemma}

\begin{proof}
(a) From (\ref{equ:g_simple}), we equivalently have $A=S \circ \left(\sum_k a_k F_k \right)$ where
$F_k$ are permutation matrices with $(F_k)_{ij} = \delta_{P_{ij} k}$.
Then we have $A=\sum_k a_k E_k$ where $E_k=S \circ F_k$.
Note that from (\ref{equ:ring_bilinear}) $E_k$ can be directly derived from the indexing tensor $M$ as $(E_k)_{ij}=M_{ikj}$.

(b) Let $a=e_k$. Then $A=E_k$ since $a_j = \delta_{jk}$.
\end{proof}

\begin{theorem}
\label{suff_asso}
The multiplication of $R$ is associative if $R$ has (i) the exclusive sub-product distribution, (ii) the commutative property of multiplication, and (iii) a commutative property of permutation matrices $E_k$ such that $E_k E_j = E_j E_k, \forall j,k$.
\end{theorem}

\newcommand\myeq{\mathrel{\overset{\makebox[0pt]{\mbox{\normalfont\tiny\sffamily (ii)}}}{=}}}
\begin{proof}
Let $a,b,c\in R$ and $c=a\cdot b$.
By Lemma \ref{permu_equ}, (i), and (ii), we have the $j$-th column of the isomorphic $C$ as 
$C_{:,j}=C e_j = (a\cdot b)\cdot e_j = e_j \cdot (a\cdot b) = E_jAb = \sum_k a_k E_j E_k b$.
Similarly, $(AB)_{:,j}=AB e_j = a\cdot (b\cdot e_j) = a\cdot (e_j \cdot b) = A E_j b = \sum_k a_k E_k E_j b$.
Then, we have $C=AB$ by (iii) and thus the associative property of multiplication by Lemma \ref{as_equi}.
\end{proof}

\begin{corollary}
\label{suff_diag}
The multiplication of $R$ is associative if $R$ has the properties (i) and (ii) and the isomorphic matrix is diagonalizable over $\mathbb{R}$.
\end{corollary}

\begin{proof}
By Lemma \ref{permu_equ} and (i), for $a\in R$ we consider its isomorphic matrix $A$ which is diagonalizable: $A=\sum_k a_k E_k = T^{-1} D T$ where $D$ is a diagonal matrix in form of indeterminates $a_k$ as $D(a)$.
We can diagonalize each permutation matrix $E_j$ by setting $a=e_j$ and have $E_j=T^{-1} D(e_j) T$.
Then the commutative property of permutation matrices holds since the multiplication of diagonal matrices is commutative: $E_k E_j = T^{-1} D(e_k) T T^{-1} D(e_j) T = T^{-1} D(e_k) D(e_j) T = T^{-1} D(e_j) D(e_k) T = E_j E_k$. 
By Theorem \ref{suff_asso}, the multiplication of $R$ is thus associative.
\end{proof}

Among the proper rings found in Section \ref{ssec:proper_ring_mult}, $R_H$ and $R_{O4}$ directly have the associative property of multiplication by Corollary \ref{suff_diag} since they are diagonalizable over $\mathbb{R}$.
The rest of them ($\mathbb{C}$, $R_{H4\mbox{\scriptsize{-I}}}$, $R_{H4\mbox{\scriptsize{-II}}}$, $R_{O4\mbox{\scriptsize{-I}}}$, and $R_{O4\mbox{\scriptsize{-II}}}$) all possess the commutative property of permutation matrices and thus have the associative property as well by Theorem \ref{suff_asso}.

\subsection{Comparison with Weight Pruning for Recognition}
\label{ssec:vs_legr}

\begin{figure}[h]
\centering
\includegraphics[width=8.7cm]{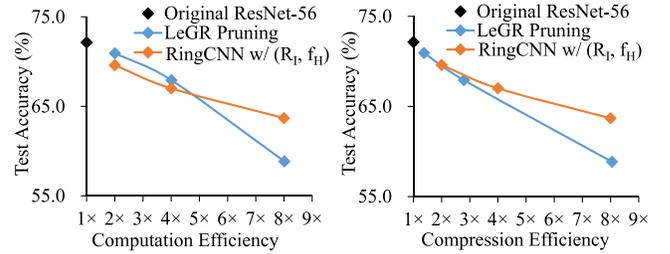} 
\caption{Computation efficiency (left) and weight compression efficiency (right) versus test accuracy for ResNet-56 on CIFAR-100 \cite{cifar10_100}.
Note that ResNet-56 for CIFAR-100 has smaller-size feature maps for deeper layers, so the efficiencies of computation and compression are not the same for LeGR.
}
\label{fig:fig_legr_comp}
\end{figure}

Recent pruning techniques for image recognition have mainly focused on structured pruning for their practical efficiency.
In Fig. \ref{fig:fig_legr_comp}, we compare RingCNN models to LeGR \cite{legr_prune_2020} which is the state-of-the-art for structured filter pruning.
RingCNN models show their potential by outperforming LeGR for high computation efficiency (left) and for a wide range of compression efficiency (right).
The study on hardware accelerators in this aspect will be our future work.

\end{document}